\documentclass[a4paper]{article}[11pts]

\usepackage[english]{babel}
\usepackage[utf8x]{inputenc}
\usepackage[T1]{fontenc}

\normalsize

\usepackage[a4paper,top=1in,bottom=1in,left=1in,right=1in,marginparwidth=1in]{geometry}

\usepackage{setspace}
\usepackage{amsmath}
\usepackage{amssymb} 
\usepackage{amsthm}
\usepackage{amsfonts, mathtools}
\usepackage{algorithm,algpseudocode}
\usepackage{bm}
\usepackage{bbm}
\usepackage{comment}
\usepackage{graphicx}
\usepackage{multirow, booktabs}
\usepackage{caption}
\usepackage{subcaption}
\usepackage[colorinlistoftodos]{todonotes}
\usepackage[colorlinks=true, allcolors=blue]{hyperref}

\newcommand{\E}{\mathbb{E}}

\DeclareMathOperator*{\argmin}{arg\,min}

\newtheorem{proposition}{Proposition}
\newtheorem{lemma}{Lemma}

\newtheorem{assumption}{Assumption}
\newtheorem{theorem}{Theorem}
\newtheorem{corollary}{Corollary}
\newtheorem{definition}{Definition}
\newtheorem{claim}{Claim}

\usepackage{natbib}

\usepackage{natbib}

\title{Non-stationary Bandits with Knapsacks}
\author{}
\date{}

\author{Shang Liu$^\dagger$, Jiashuo Jiang$^\ddagger$, Xiaocheng Li$^\dagger$}
\date{\small 
$^\dagger$Imperial College Business School, Imperial College London\\
$^\ddagger$NYU Stern School of Business\\
s.liu21@imperial.ac.uk, jj2398@stern.nyu.edu, xiaocheng.li@imperial.ac.uk}

\begin{document}
\maketitle

\onehalfspacing

\begin{abstract}
In this paper, we study the problem of bandits with knapsacks (BwK) in a non-stationary environment. The BwK problem generalizes the multi-arm bandit (MAB) problem to model the resource consumption associated with playing each arm. At each time, the decision maker/player chooses to play an arm, and s/he will receive a reward and consume certain amount of resource from each of the multiple resource types. The objective is to maximize the cumulative reward over a finite horizon subject to some knapsack constraints on the resources. Existing works study the BwK problem under either a stochastic or adversarial environment. Our paper considers a non-stationary environment which continuously interpolates these two extremes. We first show that the traditional notion of variation budget is insufficient to characterize the non-stationarity of the BwK problem for a sublinear regret due to the presence of the constraints, and then we propose a new notion of global non-stationarity measure. We employ both non-stationarity measures to derive upper and lower bounds for the problem. Our results are based on a primal-dual analysis of the underlying linear programs and highlight the interplay between the constraints and the non-stationarity. Finally, we also extend the non-stationarity measure to the problem of online convex optimization with constraints and obtain new regret bounds accordingly.
\end{abstract}

\section{Introduction}
\label{sec_intro}

The multi-armed bandit (MAB) problem characterizes a problem for which a limited amount of resource must be allocated between competing (alternative) choices in a way that maximizes the expected gain. The \textit{bandits with knapsacks} (BwK) problem generalizes the multi-armed bandits problem to allow more general resource constraints structure on the decisions made over time, in addition to the customary limitation on the time horizon. Specifically, for the BwK problem, the decision maker/player chooses to play an arm at each time period; s/he will receive a reward and consume certain amount of resource from each of the multiple resource types. Accordingly, the objective is to maximize the cumulative reward over a finite time horizon and subject to an initial budget of multiple resource types. The BwK problem was first introduced by \cite{badanidiyuru2013bandits} as a general framework to model a wide range of applications, including dynamic pricing and revenue management \citep{besbes2012blind}, Adwords problem \citep{mehta2005adwords} and more. 

The standard setting of the BwK problem is stochastic where the joint distribution of reward and resource consumption for each arm remains stationary (identical) over time. Under such setting, a linear program (LP), that takes the expected reward and resource consumption of each arm as input, both serves as the benchmark for regret analysis and drives the algorithm design \citep{badanidiyuru2013bandits, agrawal2014bandits}. Notably, a static best distribution prescribed by the LP's optimal solution is used for defining the regret benchmark. An alternative setting is the adversarial BwK problem where the reward and the consumption may no long follow a distribution and they can be chosen arbitrarily over time. Under the adversarial setting, a sublinear regret is not achievable in the worst case; \cite{immorlica2019adversarial} derive a $O(\log T)$ competitive ratio against the static best distribution benchmark which is aligned with the static optimal benchmark in the adversarial bandits problem \citep{auer1995gambling}. Another key of the BwK problem is the number of resource types $d$. When $d=1$, one optimal decision is to play the arm with largest (expected) reward to (expected) resource consumption ratio, where the algorithm design and analysis can be largely reduced to the MAB problem. When $d>1$, the optimal decision in general requires to play a combination of arms (corresponding the optimal basis of the underlying LP). \cite{rangi2018unifying} focus on the case of $d=1$ and propose an EXP3-based algorithm that attains a regret of $O(\sqrt{m B \log m})$ against the best fixed distribution benchmark. Their result thus bridges the gap between the stochastic BwK problem and the adversarial BwK problem for the case of $d=1$. The difference between the cases of $d=1$ and $d>1$ is also exhibited in the derivation of problem-dependent regret bounds for the stochastic BwK problem \citep{flajolet2015logarithmic, li2021symmetry, sankararaman2021bandits}.

In this paper, we study the non-stationary BwK problem where the reward and the resource consumption at each time are sampled from a distribution as the stochastic BwK problem but the distribution may change over time. The setting relaxes the temporally i.i.d. assumption in the stochastic setting and it can be viewed as a soft measure of adversity. We aim to relate the non-stationarity (or adversity) of the distribution change with the best-achievable algorithm performance, and thus our result bridges the two extremes of BwK problem: stochastic BwK and adversarial BwK. We consider a dynamic benchmark to define the regret; while such a benchmark is aligned with the dynamic benchmark in other non-stationary learning problem \citep{besbes2014stochastic, besbes2015non, cheung2019learning, faury2021regret}, it is stronger than the static distribution benchmark in adversarial BwK \citep{rangi2018unifying, immorlica2019adversarial}. Importantly, we use simple examples and lower bound results to show that the traditional non-stationarity measures such as change points and variation budget are not suitable for the BwK problem due to the presence of the constraints. We introduce a new non-stationarity measure called \textit{global variation budget} and employ both of this new measure and the original variation budget to capture the underlying non-stationarity of the BwK problem. We analyze the performance of a sliding-window UCB-based BwK algorithm and derive a near-optimal regret bound. Furthermore, we show that the new non-stationarity measure can also be applied to the problem of \textit{online convex optimization with constraints} (OCOwC) and extend the analyses therein.

\subsection{Related literature}

The study of non-stationary bandits problem begins with the change-point or piecewise-stationary setting where the distribution of the rewards remains constant over epochs and changes at unknown time instants \citep{garivier2008upper, yu2009piecewise}. The prototype of non-stationary algorithms such as discounted UCB and sliding-window UCB are proposed and analyzed in \citep{garivier2008upper} to robustify the standard UCB algorithm against the environment change. The prevalent variation budget measure $V=\sum_{t=1}^{T-1}\|\mathcal{P}_t-\mathcal{P}_{t+1}\|$ (where $\mathcal{P}_t$ and the norm bear different meaning under different context) is later proposed and widely studied under different contexts, such as non-stationary stochastic optimization (\cite{besbes2015non}), non-stationary MAB (\cite{besbes2014stochastic}), non-stationary linear bandits (\cite{cheung2019learning}), and non-stationary generalized linear bandits (\cite{faury2021regret}) problems. In general, these works derive lower bound of $\Omega(V^{\frac{1}3} T^{\frac{2}3})$, and propose algorithms that achieve near-optimal regret of $\tilde{O}(V^{\frac{1}3} T^{\frac{2}3}).$ \cite{cheung2019learning} and \cite{faury2021regret} require various assumption on the decision set to attain such upper bound; under more general conditions, a regret bound of $\tilde{O}(V^{\frac{1}5} T^{\frac{4}5})$ can be obtained \citep{faury2021regret}. With the soft measure of non-stationarity, the existing results manage to obtain sublinear regret bounds in $T$ against dynamic optimal benchmarks. In contrast, a linear regret in $T$ is generally inevitable against the dynamic benchmark when the underlying environment is adversarial. We remark that while all these existing works consider the unconstrained setting, our work complements this line of literature with a proper measure of non-stationarity in the constrained setting.

Another related stream of literature is the problem of online convex optimization with constraints (OCOwC) which extends the OCO problem in a constrained setting. There are two types of constraints considered: the long-term constraint  \citep{jenatton2016adaptive,neely2017online} and the cumulative constraint \citep{yuan2018online,yi2021regret}. The former defines the constraint violation by $\|(\sum_{t=1}^T \bm g_t(x_t))^+\|$ whilst the latter defines it by $\sum_{t=1}^T\|(\bm g_t(x_t))^+\|$ where $(\cdot)^+$ is the positive-part function. The existing works mainly study the setting where $\bm{g}_t=\bm{g}$ for all $t$ and $\bm{g}$ is known a priori. \cite{neely2017online} considers a setting where $\bm{g}_t$ is i.i.d. generated from some distribution. In this paper, we show that our non-stationarity measure naturally extends to this problem and derives bounds for OCOwC when $\bm{g}_t$'s are generated in a non-stationary manner.

A line of works in operations research and operations management literature also study non-stationary environment for online decision making problem under constraints \citep{ma2020approximation, freund2019good, jiang2020online}. The underlying problem along this line can be viewed as a full-information setting where at each time $t$, the decision is made after the observation of the function/realized randomness/customer type, while the BwK and OCOwC can be viewed as a partial-information setting where the decision is made prior to and may affect the observation. So for the setting along this line, there is generally no need for exploration in algorithm design, and the main challenge is to trade off the resource consumption with the reward earning. 
\section{Problem Setup}
\label{sec_prel}

We first introduce the formulation of the BwK problem. The decision-maker/learner is given a fixed finite set of arms $\mathcal{A}$ (with $|\mathcal{A}| = m$) called as \emph{action set}. There are $d$ knapsack constraints with a known initial budget of $B_j$ for $j\in[d]$. Without loss of generality, we assume $B_j=B$ for all $j.$ There is a finite time horizon $T$, which is also known in advance. At each time $t=1,...,T$, the learner must choose either to play an arm $i_t$ or to do nothing but wait. If the learner plays the arm $i$ at time $t$, s/he will receive a reward $r_{t, i} \in [0,1]$ and consume $c_{t,j,i} \in [0,1]$ amount of each resource $j$ from the initial budget $B$. As the convention, we introduce a \emph{null arm} to model ``doing nothing'' which generates a reward of zero and consumes no resource at all. We assume $(\bm{r}_t,\bm{c}_t)$ is sampled from some distribution $\mathcal{P}_t$ independently over time where $\bm{r}_t=\{r_{t,i}\}_{i\in[m]}$ and $\bm{c}_t=\{c_{t,j,i}\}_{i\in[m], j\in[d]}$. In the stochastic BwK problem, the distribution $\mathcal{P}_t$ remains unchanged over time, while in the adversarial BwK problem, $\mathcal{P}_t$ is chosen adversarially. In our paper, we allow $\mathcal{P}_t$ to be chosen adversarially, while we use some non-stationarity measure to control the extent of adversity in 
choosing $\mathcal{P}_t$'s.

At each time $t$, the learner needs to pick $i_t$ using the past observations until time $t-1$ but without observing the outcomes of time step $t$. The resource constraints are assumed to be \emph{hard} constraints, i.e., the learner must stop at the earliest time $\tau$ when at least one constraint is violated, i.e. $ \sum_{t=1}^{\tau} c_{t,j,i_t} > B$, or the time horizon $T$ is exceeded. The objective is to maximize the expected cumulative reward until time $\tau$, i.e. $\mathbb{E}[\sum_{t=1}^{\tau - 1} r_{t,i_t}]$. To measure the performance of a learner, we define the regret of the algorithm/policy $\pi$ adopted by the learner as 
$$\mathrm{Reg}(\pi, T) \coloneqq \mathrm{OPT}(T) - \mathbb{E}\left[\sum_{t=1}^{\tau - 1} r_{t,i_t} \middle| \pi\right].$$
Here $\text{OPT}(T)$ denotes the expected cumulative reward of the optimal dynamic policy given all the knowledge of $\mathcal{P}_t$'s in advance. Its definition is based on the dynamic optimal benchmark which allows the arm play decisions/distributions to change over time. As a result, it is stronger than the optimal fixed distribution benchmark used in the adversarial BwK setting \citep{rangi2018unifying, immorlica2019adversarial}.

\subsection{A Motivating Example}

\label{motivate_eg}

The conventional variation budget is defined by 
$$V_T \coloneqq \sum_{t=1}^{T-1} \text{dist}(\mathcal{P}_t,\mathcal{P}_{t+1}).$$
By twisting the definition of the metric $\text{dist}(\cdot,\cdot)$, it captures many of the existing non-stationary measures for unconstrained learning problems. Now we use a simple example to illustrate why $V_T$ no longer fits for the constrained setting. Similar examples have been used to motivate algorithm design and lower bound analysis in \citep{golrezaei2014real, cheung2019learning,jiang2020online}, but have not been yet be exploited in a partial-information setting such as bandits problems.

Consider a BwK problem instance that has two arms (one actual arm and one null arm), and a single resource constraint with initial capacity of $\frac{T}{2}$. Without loss of generality, we assume $T$ is even. The null arm has zero reward and zero resource consumption throughout the horizon, and the actual arm always consumes 1 unit of resource (deterministically) for each play and outputs 1 unit of reward (deterministically) for the first half of the horizon, i.e., when $t=1,...,\frac{T}{2}.$ For the second half of the horizon $t=\frac{T}{2}+1,...,T$, the reward of the actual arm will change to either $1+\Delta$ or $1-\Delta$, and the change happens adversarially. For this problem instance, the distribution $\mathcal{P}_t$ only changes once, i.e., $V_T = \Delta$ (varying up to constant due to the metric definition). But for this problem instance, a regret of $\frac{T\cdot \Delta }{4}$ is inevitable. To see this, if the player plays the actual arm no less than $\frac{T}{4}$ times, then the distributions of the second half can adversarially change to the reward $1+\Delta$, and this will result in a $\frac{T\cdot \Delta }{4}$ regret at least. The same for the case of playing the actual arm for the case of no more than $\frac{T}{4}$ times, and we defer the formal analysis to the proof of the lower bounds in Theorem \ref{thm:lower}. 

This problem instance implies that a sublinear dependency on $T$ cannot be achieved with merely the variation budget $V_T$ to characterize the non-stationarity. Because with the presence of the constraint(s), the arm play decisions over time are all coupled together not only through the learning procedure, but also through the ``global'' resource constraint(s). For the unconstrained problems, the non-stationarity affects the effectiveness of the learning of the system; for the constrained problems, the non-stationarity further challenges the decision making process through the lens of the constraints. 

\subsection{Non-stationarity Measure and Linear Programs}

We denote the expected reward vector as $\bm{\mu}_t=\{\mu_{t,i}\}_{i\in[m]}$ and the expected consumption matrix as $\bm{C}_t=\{C_{t,j,i}\}_{j\in[d], i\in[m]}$, i.e.,
$$\mu_{t,i} \coloneqq \mathbb{E}[r_{t,i}], \quad C_{t,j,i} \coloneqq \mathbb{E}[c_{t,j,i}].$$

We first follow the conventional variation budget and define the \textit{local non-stationarity budget}:
\footnote{Throughout the paper, for a vector $\bm v \in \mathbb{R}^n$, we denote its $L_1$ norm and $L_\infty$ norm by
$ \|\bm v\|_{1} \coloneqq \sum_{i=1}^n |v_i|, $
$ \|\bm v\|_{\infty} \coloneqq \max_{1\leq i\leq n} |v_i|. $
For a matrix $\bm M \in \mathbb{R}^{m\times n}$, we denote its $L_1$ norm and $L_\infty$ norm by
$\|\bm M\|_{1} \coloneqq \sup_{\bm x\neq 0} \frac{\|\bm M \bm x\|_{1}}{\|\bm x\|_1} = \max_{1\leq j \leq n}\sum_{i=1}^m |M_{ij}|,$
$\|\bm M\|_{\infty} \coloneqq \sup_{\bm x \neq 0} \frac{\|\bm M \bm x\|_{\infty}}{\|\bm x\|_{\infty}} = \max_{1\leq i \leq m} \sum_{j=1}^n |M_{ij}|. $}
\begin{align*}
V_1 & \coloneqq \sum_{t = 1}^{T-1} \| \bm{\mu}_t - \bm \mu_{t+1} \|_{\infty},\\
V_{2,j} & \coloneqq \sum_{t = 1}^{T-1} \| \bm{C}_{t,j} - \bm{C}_{t+1,j} \|_{\infty},\ \ V_2 \coloneqq \max_{1 \le j\leq d} V_{2,j}.
\end{align*}
We refer to the measure as a local one in that they capture the local change of the distributions between time $t$ and time $t+1$. 

Next, we define the \emph{global non-stationarity budget}:
\begin{align*}
    W_1 &\coloneqq \sum_{t = 1}^T \| \bm{\mu}_t - \bar{\bm{\mu}} \|_{\infty},\\
    W_2 &\coloneqq \sum_{t = 1}^T \| \bm{C}_{t} - \bar{\bm{C}} \|_1,
\end{align*}
where $\bar{\bm{\mu}}=\frac{1}{T}\sum_{t=1}^T \bm{\mu}_t$ and $\bar{\bm{C}}=\frac{1}{T}\sum_{t=1}^T \bm{C}_t$. These measures capture the total deviations for all the $\bm{\mu}_t$'s and $\bm{C}_t$ from their global averages. By definition, $W_1$ and $W_2$ upper bound $V_1$ and $V_{2}$ (up to a constant), so they can be viewed as a more strict measure of non-stationarity than the local budget. In the definition of $W_2$, the L$_1$ norm is not essential and it aims to sharpen the regret bounds (by corresponding to the upper bound on the dual optimal solution in supremum norm to be defined shortly).

All the existing analyses of the BwK problem utilize the underlying linear program (LP) and establish the LP's optimal objective value as an upper bound of the regret benchmark OPT$(T)$. In a non-stationary environment, the underlying LP is given by
\begin{align*}
    \mathrm{LP}\left(\{\bm \mu_t\}, \{\bm C_t\}, T\right) \ \coloneqq \ & \max_{\bm x_1,\dots,\bm x_T} \ \sum_{t=1}^T \bm \mu_t^\top \bm x_t \\
    & \text{s.t. } \sum_{t=1}^T \bm C_t \bm x_t \le \bm B,\quad \bm x_t \in \Delta_m,\ t = 1,\dots, T,
\end{align*}
where $\bm{B}=(B,...,B)^\top$ and $\Delta_m$ denotes the $m$-dimensional standard simplex. We know that $$\mathrm{LP}(\{\bm \mu_t\}, \{\bm C_t\}, T) \geq \mathrm{OPT(T)}.$$
In the rest of our paper, we will use $\mathrm{LP}(\{\bm \mu_t\}, \{\bm C_t\}, T)$ for the analysis of regret upper bound. We remark that in terms of this LP upper bound, the dynamic benchmark allows the $\bm{x}_t$ to take different values, while the static benchmark will impose an additional constraint to require all the $\bm{x}_t$ be the same.

For notation simplicity, we introduce the following linear growth assumption. All the results in this paper still hold without this condition.

\begin{assumption}[Linear Growth]
\label{assumption:linear}
We have the resource budget $B=bT$ for some $b>0$.
\end{assumption}

Define the single-step LP by
\begin{align*}
    \mathrm{LP}(\bm \mu, \bm C) \ \coloneqq \ & \max_{\bm x} \ \bm \mu^\top \bm x \\
    & \text{s.t. } \bm C \bm x \le \bm{b}, \quad \bm x \in \Delta_m.
\end{align*}
where $\bm{b}=(b,...,b)^\top.$ The single-step LP's optimal objective value can be interpreted as the single-step optimal reward under a normalized resource budget $\bm{b}$. 

Throughout the paper, we will use the dual program and the dual variables to relate the resource consumption with the reward, especially for the non-stationary environment. The dual of the benchmark LP$(\{\bm \mu_t\}, \{\bm C_t\}, T)$ is
\begin{align*}
    \text{DLP}(\{\bm \mu_t\}, \{\bm C_t\}) \coloneqq \min_{\bm q, \bm \alpha} \ & T \cdot \bm{b}^\top \bm q + \sum_{t=1}^T \alpha_t\\
    \text{s.t. }& \bm \mu_t - \bm C_t^\top \bm q - \alpha_t \cdot \mathbf{1}_m \leq 0, \quad t = 1,\dots, T,\\
    & \bm q \ge 0
\end{align*}
where $\mathbf{1}_m$ denotes an $m$-dimensional all-one vector. Here we denotes one optimal solution as $(\bm{q}^*, \bm{\alpha}^*).$

The dual of the single-step LP$(\bm \mu_t, \bm C_t)$ is
\begin{align*}
    \text{DLP}(\bm \mu_t, \bm C_t) \coloneqq \min_{\bm q, \alpha} \ & \bm b^\top \bm q + \alpha\\
    \text{s.t. }& \bm \mu_t - \bm C_t^\top \bm q - \alpha \cdot \mathbf{1}_m \leq 0,\\
    & \bm q \ge 0.
\end{align*}
Here we denotes one optimal solution as $(\bm{q}^*_t, \alpha_t^*).$ We remark that these two dual LPs are always feasible by choosing $\bm{q}=0$ and some large $\alpha$, so there always exists an optimal solution.

The dual optimal solutions $\bm{q}^*$ and $\bm{q}^*_t$ are also known as the dual price, and they quantify the cost efficiency of each arm play. 

Define 
$$\bar{q} = \max\left\{\|\bm{q}^*\|_\infty, \|\bm{q}^*_t\|_\infty, t=1,...,T\right\}.$$
The quantity $\bar{q}$ captures the maximum amount of achievable reward by each unit of resource consumption. We will return with more discussion on this quantity $\bar{q}$ after we present the regret bound. 

\begin{lemma}
\label{lemma:barq_upperbound}
We have the following upper bound on $\bar{q},$
$$ \bar{q} \leq \frac{1}{b}. $$
\end{lemma}



\begin{proposition}
\label{prop:LP}
We have
$$\sum_{t=1}^T \mathrm{LP}(\bm \mu_t, \bm C_t) \leq \mathrm{LP}(\{\bm \mu_t\}, \{\bm C_t\}, T) \leq T \cdot \mathrm{LP}(\bar{\bm \mu}, \bar{\bm C}) + W_1 + \bar{q} W_2 \leq \sum_{t=1}^T \mathrm{LP}(\bm \mu_t, \bm C_t) + 2 (W_1 + \bar{q} W_2). $$
\end{proposition}

Proposition \ref{prop:LP} relates the optimal value of the benchmark $\mathrm{LP}(\{\bm \mu_t\}, \{\bm C_t\}, T)$ with the optimal values of the single-step LPs. To interpret the bound, $\mathrm{LP}(\{\bm \mu_t\}, \{\bm C_t\}, T)$ works as an upper bound of the OPT$(T)$ in defining the regret, and the summation of $\mathrm{LP}(\bm \mu_t, \bm C_t)$ corresponds to the total reward obtained by evenly allocating the resource over all time periods. In a stationary environment, these two are the same as the optimal decision naturally corresponds to an even allocation of the resources. In a non-stationary environment, it can happen that the optimal allocation of the resource corresponds an uneven one for $\mathrm{LP}(\{\bm \mu_t\}, \{\bm C_t\}, T)$. For the problem instance in Section \ref{motivate_eg}, the optimal allocation may be either to exhaust all the resource in first half of time periods or preserve all the resource for the second half. In such case, forcing an even allocation will reduce the total reward obtained. The proposition tells that the reduction can be bounded by $2W_1+2\bar{q}W_2$ where the non-stationarity in resource consumption $W_2$ is weighted by the dual price upper bound $\bar{q}.$

\section{Sliding-Window UCB for Non-stationary BwK}
\label{sec_sl}

In this section, we adapt the standard sliding-window UCB algorithm for the BwK problem (Algorithm \ref{alg:sl}) and derive a near-optimal regret bound. The algorithm will terminate when any type of the resources is exhausted. At each time $t$, it constructs standard sliding-window confidence bounds for the reward and the resource consumption. Specifically, we define the sliding-window estimators by
\begin{align*}
\hat{\mu}_{t,i}^{(w)} & \coloneqq \frac{\sum_{s = 1\vee(t-w)}^{t-1} r_{t,j} \cdot \mathbbm{1}\{i_s = i\}}{n_{t,i}^{(w)} + 1}, \\ \hat{C}_{t,j,i}^{(w)} & \coloneqq \frac{\sum_{s = 1\vee(t-w)}^{t-1} c_{t,j,i} \cdot \mathbbm{1}\{i_s = i\}}{n_{t,i}^{(w)} + 1},   
\end{align*}
where $n_{t,i}^{(w)} = \sum_{s = 1\vee(t-w)}^{t-1}  \mathbbm{1}\{i_s = i\}$ denotes the number of times that the $i$-th arm has been played in the last $w$ time periods. To be optimistic on the objective value, UCBs are computed for rewards and LCBs are computed for the resource consumption, respectively. With the confidence bounds, the algorithm solves a single-step LP to prescribe a randomized rule for the arm play decision.

Our algorithm can be viewed as a combination of the standard sliding-window UCB algorithm \citep{garivier2008upper, besbes2015non} with the UCB for BwK algorithm \citep{agrawal2014bandits}. It makes a minor change compared to \citep{agrawal2014bandits} which solves a single-step LP with a shrinkage factor $(1-\epsilon)$ on the right-hand-side. The shrinkage factor therein ensures that the resources will not be exhausted until the end of the horizon, but it is not essential to solving the problem. For simplicity, we choose the more natural version of the algorithm which directly solves the single-step LP. We remark that the knowledge of the initial resource budget $\bm{B}$ and the time horizon $T$ will only be used for defining the right-hand-side of the constraints for this $\mathrm{LP}(\mathrm{UCB}_t(\bm \mu_t), \mathrm{LCB}_t(\bm C_t))$.

\begin{algorithm}[ht!]
\caption{Sliding-Window UCB Algorithm for BwK}
\label{alg:sl}
\begin{algorithmic}[1]
\Require Initial resource budget $\bm{B}$, time horizon $T$, window sizes $w_1$ (for reward) and $w_2$ (for resource consumption).
\Ensure Arm play indices $\{i_t\}$'s
\While{$t \leq T$}
    \If{$\sum_{s=1}^{t-1} c_{t,j} > B$ for some $j$}
        \State Break
        \State \textcolor{blue}{\%\% Terminate the procedure if any resource is exhausted.}
    \EndIf
    \State Construct confidence bounds $\mathrm{UCB}_t(\bm \mu_t), \mathrm{LCB}_t(\bm C_t)$ with window size $w_1, w_2$
    \begin{align*}
        \mathrm{UCB}_{t,i}(\bm \mu_t) & \coloneqq \hat{\mu}_{t,i}^{(w_1)} + \sqrt{\frac{2}{n_{t,i}^{(w_1)} + 1} \log(12 m T^3)} \\ \mathrm{LCB}_{t,j,i}(\bm C_t)& \coloneqq \hat{C}_{t,j,i}^{(w_{2})} - \sqrt{\frac{2}{n_{t,i}^{(w_2)} + 1} \log(12 m d T^3)} \end{align*}
    \State Solve the single-step problem $\mathrm{LP}(\mathrm{UCB}_t(\bm \mu_t), \mathrm{LCB}_t(\bm C_t))$
    \State Denote its optimal solution by $\bm x_t^*=(x_{t,1}^*,...,x_{t,m}^*)$
    \State Pick arm $i_t$ randomly according to $\bm x_t^*$, i.e., $\mathbb{P}(i_t = i) = x_{t,i}^*$
    \State Observe the realized reward $r_t$ and resource consumption $c_{t,j}$ for $\ j\in[d]$
\EndWhile
\end{algorithmic}
\end{algorithm}


Now we begin to analyze the algorithm's performance. For starters, the following lemma states a standard concentration result for the sliding-window confidence bound.

\begin{lemma}
\label{lemma:confidence_bound}
The following inequalities hold for all $t=1,...,T$ with probability at least $1 - \frac{1}{3T}$:
$$ \mathrm{UCB}_{t,i}(\bm \mu_t) + \sum_{s=1\vee(t-w_1)}^{t-1} \|\bm \mu_{s} - \bm\mu_{s+1}\|_{\infty} \geq\mu_{t,i}, \quad \forall i, $$
$$ \mathrm{LCB}_{t,j,i}(\bm C_t) - \sum_{s=1\vee(t-w_{2})}^{t-1} \|\bm C_{s,j} - \bm C_{s+1,j}\|_{\infty}  \leq C_{t,j,i}, \quad \forall j,i$$
where the UCB and LCB estimators are defined in Algorithm \ref{alg:sl}.
\end{lemma}

With Lemma \ref{lemma:confidence_bound}, we can employ a concentration argument to relate the realized reward (or resource consumption) with the reward (or resource consumption) of the LP under its optimal solution. In Lemma \ref{lemma:reward_consumption}, recall that $\tau$ is the termination time of the algorithm where some type of resources is exhausted, and $\bm{x}_s^*$ is defined in Algorithm \ref{alg:sl} as the optimal solution of the LP solved at time $s$.

\begin{lemma}
\label{lemma:reward_consumption}
For Algorithm \ref{alg:sl}, the following inequalities hold for all $t \leq \min\{\tau, T\}$,
$$ \left|\sum_{s=1}^t (r_t - \mathrm{UCB}_s(\bm \mu_s)^\top \bm x_s^*)\right| \leq 4 \sqrt{T} \log (12 mT^3) + 8 \sqrt{2\log(12 mT^3) m} \cdot \frac{T}{\sqrt{w_1}} + w_1 V_1, $$
$$ \left|\sum_{s=1}^t (c_{s,j} - \mathrm{LCB}_t(\bm C_{s,j})^\top \bm x_s^*)\right| \leq 4\sqrt{T}\log(12 md T^3) + 8 \sqrt{2\log(12 md T^3) m} \cdot \frac{T}{\sqrt{w_2}} + w_2 V_2 \text{ for all } j,$$
with probability at least $1 - \frac{1}{T}.$
\end{lemma}

We note that the single-step LP's optimal solution  is always subject to the resource constraints. So the second group of inequalities in Lemma \ref{lemma:reward_consumption} implies the following bound on the termination time $\tau$. Recall that $b$ is the resource budget per time period; for a larger $b$, the resource consumption process becomes more stable and the budget is accordingly less likely to be exhausted too early.

\begin{corollary}
\label{coro:terminate}
If we choose $w_2 = \min\left\{ \lceil m ^{\frac{1}{3}} V_2^{-\frac{2}{3}} T^{\frac{2}{3}}  \log^{\frac{1}{3}}(12 md T^3) \rceil, T\right\}$ in Algorithm \ref{alg:sl}, the following inequality holds 
\begin{align*}
    T-\tau &\le \frac{1}{b}\cdot\left(14 m ^{\frac{1}{3}} V_2^{\frac{1}{3}} T^{\frac{2}{3}}  \log^{\frac{1}{3}}(12mdT^3)+ 8 \sqrt{2mT} \sqrt{\log(12mdT^3)}+4\sqrt{T}\log(12 md T^3)\right) \\
    &= \tilde{O}\left(\frac{1}{b}(m^{1/3}V_2^{\frac{1}{3}} T^{\frac{2}{3}}+\sqrt{mT})\right)
\end{align*}
with probability at least $1 - \frac{1}{2T}.$
\end{corollary}

To summarize, Lemma \ref{lemma:reward_consumption} compares the realized reward with the cumulative reward of the single-step LPs, and Corollary \ref{coro:terminate} bounds the termination time of the algorithm. Recall that Proposition \ref{prop:LP} relates the cumulative reward of the single-step LPs with the underlying LP -- the regret benchmark. Putting together these results, we can optimize $w_1$ and $w_2$ by choosing
\begin{align*}
w_1  = \min\left\{\lceil m^{\frac{1}{3}} V_1^{-\frac{2}{3}} T^{\frac{2}{3}}  \log^{\frac{1}{3}}(12mT^3)\rceil , T\right\}, \quad
w_2  = \min\left\{\lceil m^{\frac{1}{3}} V_2^{-\frac{2}{3}} T^{\frac{2}{3}}  \log^{\frac{1}{3}}(12 md T^3)\rceil , T\right\}
\end{align*}
and then obtain the final regret upper bound as follows.

\begin{theorem}
\label{thm:upper}
Under Assumption \ref{assumption:linear}, the regret of Algorithm \ref{alg:sl} is upper bounded as
\begin{align*}
    \mathrm{Reg}(\pi_1, T) &
\leq \frac{1}{b}\left(4\sqrt{T}\log(12 md T^3) + (14 + 2d) m ^{\frac{1}{3}}V_2^{\frac{1}{3}} T^{\frac{2}{3}}  \log^{\frac{1}{3}}(12mdT^3) +8\sqrt{2mT} \sqrt{\log(12mdT^3)}+ 1\right)\\
& \ \ \ + 4\sqrt{T}\log(12 m T^3) + 16 m ^{\frac{1}{3}} V_1^{\frac{1}{3}} T^{\frac{2}{3}}  \log^{\frac{1}{3}}(12mT^3) + 2(W_1 + \bar{q} W_2)\\    
  &=  \tilde{O}\left(\frac{1}{b}\sqrt{mT} + m^{\frac{1}{3}} V_1^{\frac{1}{3}} T^{\frac{2}{3}}  +  \frac{1}{b} \cdot m^{\frac{1}{3}}d V_2^{\frac{1}{3}} T^{\frac{2}{3}} + W_1 + \bar{q} W_2\right)
\end{align*}
where $\pi_1$ denotes the policy specified by Algorithm \ref{alg:sl} and $\tilde{O}(\cdot)$ hides the universal constant and the logarithmic factors.
\end{theorem}

Theorem \ref{thm:upper} provides a regret upper bound for Algorithm \ref{alg:sl} that consists of several parts. The first part of the regret bound is on the order of $\frac{1}{b}\sqrt{mT}$ and it captures the regret when the underlying environment is stationary. The remaining parts of the regret bound characterize the relation between the intensity of non-stationarity and the algorithm performance. The non-stationarity from both the reward and the resource consumption will contribute to the regret bound and that from the resource consumption will be weighted by a factor of $\frac{1}{b}$ or $q$ (See Lemma \ref{lemma:barq_upperbound} for the relation between these two). For the local non-stationarity $V_1$ and $V_2$, the algorithm requires a prior knowledge of them to decide the window length, aligned with the existing works on non-stationarity in unconstrained settings. For the global non-stationarity $W_1$ and $W_2$, the algorithm does not require any prior knowledge and they will contribute additively to the regret bound. Together with the lower bound results in Theorem \ref{thm:lower}, we argue that the regret bound cannot be further improved even with the knowledge of $W_1$ and $W_2$.

When the underlying environment degenerates from a non-stationary one to a stationary one, all the terms related to $V_1$, $V_2,$ $W_1$ and $W_2$ will disappear and then the upper bound in Theorem \ref{thm:upper} matches the regret upper bound for the stochastic BwK setting. In Theorem \ref{thm:upper}, we choose to represent the upper bound in terms of $b$ and $T$ so as to reveal its dependency on $T$ and draw a better comparison with the literature on unconstrained bandits problem. We provide a second version of Theorem \ref{thm:upper} in Appendix \ref{sec_Alternative} that matches the existing high probability bounds using OPT$(T)$ \citep{badanidiyuru2013bandits, agrawal2014bandits}. In contrast to the $\Theta(\log T)$-competitiveness result in the adversarial BwK \citep{immorlica2019adversarial}, our result implies that with a property measure of the non-stationarity/adversity, the sliding-window design provides an effective approach to robustify the algorithm performance when the underlying environment changes from stationary to non-stationary, and the according algorithm performance will not drastically deteriorate when the intensity of the non-stationarity is small.

When the resource constraints become non-binding for the underlying LPs, the underlying environment degenerates from a constrained setting to an unconstrained setting. We separate the discussion for the two cases: (i) the benchmark LP and all the single-step LPs have only non-binding constraints; (ii) the benchmark LP have only non-binding constraints but some single-step LP have binding constraints. For case (i), the regret bound in Theorem \ref{thm:upper} will match the non-stationary MAB bound \citep{besbes2014stochastic}. For case (ii), the match will not happen and this is inevitable. We elaborate the discussion in Appendix \ref{sec_binding}.

\begin{theorem}[Regret lower bounds]
\label{thm:lower}
The following lower bounds hold for any policy $\pi$,
\begin{enumerate}
    \item[(i)] $ \mathrm{Reg}(\pi, T) = \Omega(m^{\frac{1}{3}} V_1^{\frac{1}{3}} T^{\frac{2}{3}})$.
    \item[(ii)] $\mathrm{Reg}(\pi, T) = \Omega(\frac{1}{b}\cdot m^{\frac{1}{3}} V_{2}^{\frac{1}{3}} T^{\frac{2}{3}})$.
    \item[(iii)] $ \mathrm{Reg}(\pi, T) = \Omega(W_1 + \bar{q}W_2)$.
\end{enumerate}
\end{theorem}

Theorem \ref{thm:lower} presents a few lower bounds for the problem. The first and the second lower bounds are adapted from the lower bound example in non-stationary MAB \citep{besbes2014stochastic} and the third lower bound is adapted from the motivating example in \ref{motivate_eg}. There are simple examples where each one of these three lower bounds dominates over the other two. In this sense, all the non-stationarity-related terms in the upper bound of Theorem \ref{thm:upper} are necessary including the parameters $1/b$ and $\bar{q}$. There is one gap between the lower bound and the upper bound with regard to the number of constraints $d$ in the term related to $V_2$. We leave it as future work to reduce the factor to $\log d$ with some finer analysis. Furthermore, we provide a sharper definition of the global nonstationarity measure $W_1^{\text{min}}$ and $W_2^{\text{min}}$ in replacement of $W_1$ and $W_2$ in Appendix \ref{sec_better_W}. It makes no essential change to our analysis, and the two measures coincide with each other on the lower bound problem instance. We choose to use $W_1$ and $W_2$ for presentation simplicity, while $W_1^{\text{min}}$ and $W_2^{\text{min}}$ can capture the more detailed temporal structure of the nonstationarity. The discussion leaves an open question that whether the knowledge of some additional structure of the environment can further reduce the global non-stationarity.

\section{Extension to Online Convex Optimization with Constraints}
\label{sec_exten}

In this section, we show how our notion of non-stationarity measure can be extended to the problem of online convex optimization with constraints (OCOwC). Similar to BwK, OCOwC also models a sequential decision making problem under the presence of constraints. Specifically, at each time $t$, the player chooses an action $\bm{x}_t$ from some convex set $\mathcal{X}$. After the choice, a convex cost function $f_t:\mathcal{X}\rightarrow \mathbb{R}$ and a concave resource consumption function $\bm{g}_{t}=(g_{t,1},....,g_{t,d}): \mathcal{X}\rightarrow \mathbb{R}^d$ are revealed. As in the standard setting of OCO, the functions $f_t$ is adversarially chosen and thus a static benchmark is consider and defined by
\begin{align*}
    \mathrm{OPT}(T) \coloneqq \ \min_{\bm{x}\in\mathcal{X}} \ & \sum_{t=1}^T f_t(\bm{x}) \\
    \text{s.t. }&\sum_{t=1}^T  g_{t,i}(\bm{x}) \le 0, \text{ for } i \in[d].
\end{align*}
Denote its optimal solution as $\bm{x}^*$ and its dual optimal solution as $\bm q^*$.

While the existing works consider the case when $\bm{g}_{t}$'s are static or sample i.i.d. from some distribution $\mathcal{P}.$ We consider a non-stationary setting where $\bm{g}_t$ may change adversarially over time. We define a global non-stationarity measure by
$$ W \coloneqq \sum_{t=1}^T \sum_{j=1}^d \| g_{t,j} - \bar{g}_{j}\|_{\infty}$$
where $\bar{g}_{j}=\frac{1}{T}\sum_{t=1}^T g_{t,j}$ and $\|f\|_\infty \coloneqq \sup_{\bm{x}\in\mathcal{X}} |f(\bm{x})|$.

The OCOwC problem considers the following bi-objective performance measure:
\begin{align*}
 \text{Reg}_1(\pi, T) = & \sum_{t=1}^T f_t(\bm{x}_t) - \sum_{t=1}^T f_{t}(\bm{x}^*) \\
\text{Reg}_2(\pi, T) = & \sum_{i=1}^d \left(\sum_{t=1}^T g_{t,i}(\bm{x}_t)\right)^+
\end{align*}
where $(\cdot)^+$ denotes the positive-part function and $\pi$ denotes the policy/algorithm.

In analogous to the single-step LPs, we consider an  optimization problem with more restricted constraints as
\begin{align*}
    \mathrm{OPT}^\prime (T) \coloneqq \ \min_{\bm{x}\in\mathcal{X}} \ & \sum_{t=1}^T f_t(\bm{x}) \\
    \text{s.t. }&g_{t,i}(\bm{x}) \le 0, \text{ for } t \in [T], \ i \in[d].
\end{align*}
Denote its optimal solution as $\bm{x}^{*'}$,  and its dual optimal solution as $\bm{q}^{*'}$. The following proposition relates the two optimal objective values.

\begin{assumption}
\label{assumption:slater}
We assume that Slater's condition holds for both the standard OCOwC program OPT$(T)$ and the restricted OCOwC program OPT$'(T)$. We assume that $f_t, \nabla f_t, g_{t,i}$, and $\nabla g_{t,i}$ are uniformly bounded on $\mathcal{X}$ and that $\mathcal{X}$ itself is bounded. Moreover, we assume that their dual optimal solutions are uniformly bounded by $\bar{q}$, i.e.
$$\bar{q} = \max \ \left\{\|\bm q^*\|_{\infty}, \|\bm q^{*'}\|_{\infty} \right\}.$$
\end{assumption}

The following proposition relates the two optimal objective values.

\begin{proposition}
\label{prop:exten}
For OCOwC problem, under Assumption \ref{assumption:slater}, we have
$$ 0\le \mathrm{OPT}^\prime(T) - \mathrm{OPT}(T) \leq \bar{q} W.$$
\end{proposition}

Utilizing the proposition, we can show that the gradient-based algorithm of \citep{neely2017online} achieves the following regret for the setting of OCO with non-stationary constraints. Moreover, we further extend the results and discuss in Appendix \ref{sec_proof_OCO} on an oblivious adversarial setting where $\bm{g}_t$ is sampled from some distribution $\mathcal{P}_t$ and the distribution $\mathcal{P}_t$ may change over time. 

\begin{theorem}
\label{thm:exten}
Under Assumption \ref{assumption:slater}, the Virtual Queue Algorithm of \citep{neely2017online} for any OCOwC problem (denoted by $\pi_2$) produces a decision sequence $\{\bm x_t\}$ such that
$$ \mathrm{Reg}_1(\pi_2, T) \leq O(\sqrt{T}) + \bar{q}W, $$
$$ \mathrm{Reg}_2(\pi_2, T) \leq O(d\sqrt{T}).$$
\end{theorem}

The theorem tells that the non-stationarity when measured properly will not drastically deteriorate the performance of the algorithm for the OCOwC problem as well. Moreover, the non-stationarity will not affect the constraint violation at all. Together with the results for the BwK problem, we argue that the new global non-stationarity measure serves as a proper one for the constrained online learning problems. Note that the upper and lower bounds match up to a logarithmic factor (in a worst-case sense) subject to the non-stationarity measures. The future direction can be to refine the bounds in a more instance-dependent way and to identify useful prior knowledge on the non-stationarity for better algorithm design and analysis.

\bibliographystyle{informs2014}
\bibliography{main.bib} 

\newpage

\appendix\section{Numerical Experiments}
\label{apd:num_exp}
In this section, we examine three algorithms via four numerical examples. The first algorithm is the Sliding Window-UCB (SW-UCB) algorithm presented in our paper. The second algorithm is the naive UCB algorithm without any sliding windows \citep{agrawal2014bandits}. The third algorithm is LagrangeBwK presented in \citep{immorlica2019adversarial}, which is originally proposed for the adversarial BwK problem. Note that the LagrangeBwK requires an approximation of the static best distribution benchmark. For simplicity, we put the exact value of the benchmark into the algorithm. All the regret performances are reported based on the average over 100 simulation trials. 

\subsection{Cumulative Rewards}

We first conduct two experiments and plot the cumulative reward of the three algorithms over time. 

\begin{enumerate}
    \item Example 1: One-dimensional $d=1$. A two-armed instance with \emph{one} resource constraint. $T = 10000$. $B = 5000$. The reward is set to be a constant $\mu_{t,1} = \mu_{t,2} = 0.5$. For the first half time steps $t \leq T /2$, $C_{t,1,1} = 0.5$, $C_{t,1,2} = 1.0$, while for the second half $t > T/2$, $C_{t,1,1} = 1.0$, $C_{t,1,2} = 0.5$. The dynamic optimal benchmark is to play the first arm for the first half and the second arm for the second half. Accordingly, $\text{OPT}(T) = 5000$.
    \item Example 2: Two-dimensional $d=2$. A two-armed instance with \emph{two} resource constraints. $T = 10000$. $B = 5000$. For the first half $t \leq T/2$, we set the reward to be $\mu_{t,1} = \mu_{t,2} = 0.5$. We set $C_{t,1,1} = C_{t,2,2} = 1.0$ and $C_{t,1,2} = C_{t,2,1} = 0$. For the second half $t > T/2$, we force the first arm to be sub-optimal with no reward $\mu_{t,1} = 0$ and maximum consumption $C_{t,1,1} = C_{t,2,1} = 1.0$, while changing the second arm to be optimal with $\mu_{t,2} = 0.5$ and $C_{t,1,2} = C_{t,2,2} = 0.5$. The dynamic optimal benchmark is to play both arms with equal chances for the first half and only the second arm for the second half, yielding $\text{OPT}(T) = 5000$.
\end{enumerate}
For those two examples, the cumulative rewards versus time steps are shown in Figure \ref{fig:timestep}.

\begin{figure}[h]
\begin{subfigure}{.5\textwidth}
  \centering
  \includegraphics[width=.95\linewidth]{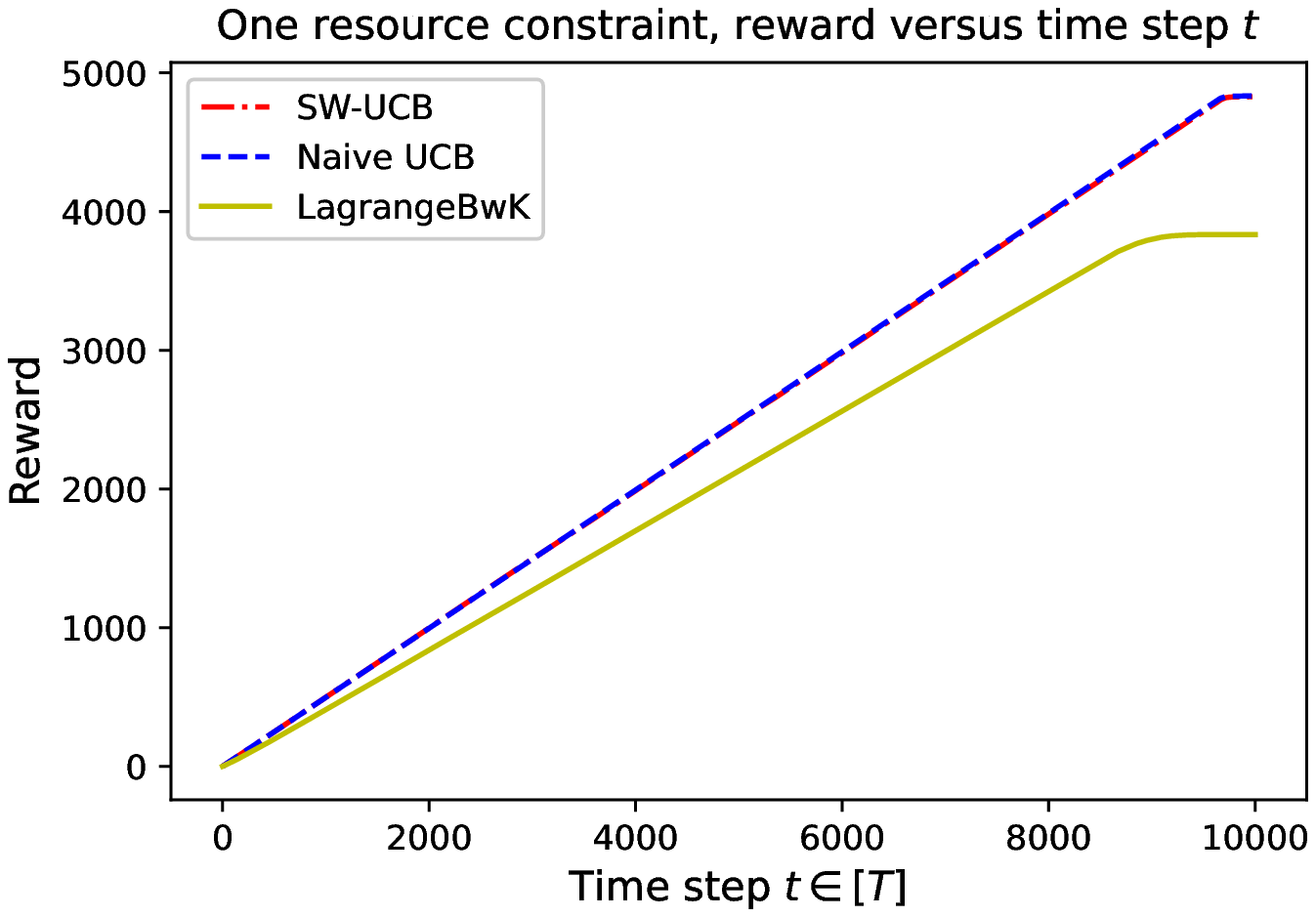}
  \caption{Example 1: One dimensional case}
  \label{fig:sfig1}
\end{subfigure}%
\begin{subfigure}{.5\textwidth}
  \centering
  \includegraphics[width=.95\linewidth]{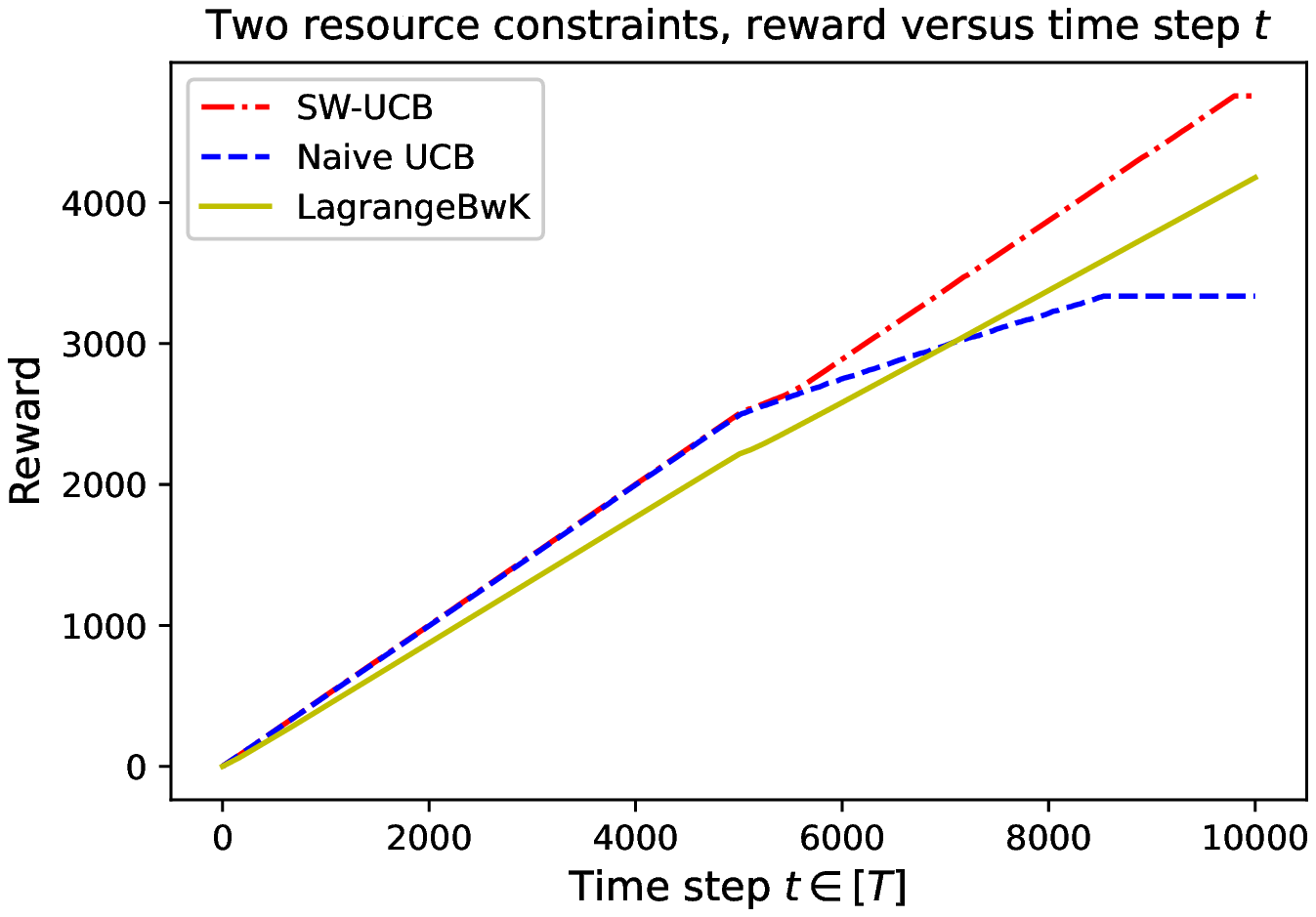}
  \caption{Example 2: Two dimensional case}
  \label{fig:sfig2}
\end{subfigure}
\caption{Cumulative rewards versus time steps.}
\label{fig:timestep}
\end{figure}

The performance of the naive UCB algorithm in one-dimensional case (see Figure \ref{fig:sfig1}) is somewhat counter-intuitive: the naive UCB algorithm originally designed for the stochastic setting performs comparable as the SW-UCB and both are better than LagrangeBwK. The reason is that the naive UCB algorithm observes the poor performance of the second arm and ends up with playing the second arm only several hundred times until the second half. So it would not take too long for the naive UCB to rectify its wrong estimate after entering the second half.

But for the two-dimensional case (see Figure \ref{fig:sfig2}), the naive UCB algorithm behaves poorly: it suffers from the abrupt change of the environment and could not adjust its approximation in time. The key difference between this and the one-dimensional setting is that here the optimal distribution for the first half requires playing both arms for sufficiently amount of time rather than simply focusing on one \emph{single} best arm for one-dimensional cases. As a result, the naive UCB algorithm accumulates linearly many observations for both arms during the first half, which significantly affects its performance during the second half. This corresponds to the slope change for the blue curve during the half way.

The algorithm designed specifically for the adversarial BwK, LagrangeBwK, performs slightly worse than our SW-UCB algorithm in both examples. This may be due to the fact that LagrangeBwK acts too conservatively for the cases that are not so \emph{adversarial}.

\subsection{Both $V$ and $W$ Matter}
In our upper bound analysis, terms that depend on $V$ and that on $W$ both appear. One may wonder: are both $V$ and $W$ necessary in the analysis? Would it be possible to reduce the terms on $V$ to $W$ or vice versa, reduce $W$ to $V$? In this subsection, we designed two examples to show that \emph{both} $V$ and $W$ can make an impact on the performance of BwK algorithms.

The non-stationary BwK problem can be factorized into two sub-problems: identifying the optimal arm distribution with respect to current environment, and finding a resource allocation rule. Larger $V$ makes the first task harder, while larger $W$ creates an obstacle for the second. The following two examples illustrate this intuition.

\begin{enumerate}
    \item Example 3: Fixed $V$, different $W$'s. A two-armed instance with two resource constraints. $T=10000$. $B=2500$. The environment has only one abrupt change point at time $\alpha T$. The local non-stationarity measure $V$ is invariant regardless of the value of $\alpha$, while $W$ depends on $\alpha$. At the first part $t\leq \alpha T$, both arms have a fixed reward $\mu_{t,1} = \mu_{t,2} = 0.5$. As for the consumption, $C_{t,1,1} = C_{t,2,2} = 0.7, \ C_{t,1,2} = C_{t,2,1} = 0.3$. At the second part $t>\alpha T$, $C_{t,1,1} = C_{t,1,2} = C_{t,2,1} = C_{t,2,2} = 1.0$, while the reward $\mu_{t,1} = 0, \ \mu_{t,2} = 0.7$. The dynamic optimal policy is to allocate all the resources to the first part and play both arms with equal chances, which leads to $\text{OPT}(T) = 2500$.
    \item Example 4: Fixed $W$, different $V$'s. A two-armed instance with two resource constraints. $T = 10000$. $B = 3125$. The time horizon is still divided into halves. The first half is stationary, while the second half is periodic but with different frequencies. At the first half $t \leq T / 2$, both arms are of fixed reward $\mu_{t,1} = \mu_{t,2} = 0.5$. The resource consumption $C_{t,1,1} = C_{t,2,2} = 1.0, \ C_{t,1,2} = C_{t,2,1} = 0$. As for the second half, the first arm now generates no reward $\mu_{t,1} = 0$ but consumes $C_{t,1,1} = C_{t,2,1} = 1.0$. The second arm's reward remains unchanged $\mu_{t,2} = 0.5$, while its consumption $C_{t,1,2} = C_{t,2,2}$ changes across the time horizon according to a piece-wise linear and periodic function ranging from $0$ to $1$. Here the global non-stationarity measure $W$ is fixed while $V$ varies with respect to the frequency. The dynamic optimal policy is to play both arms with equal chances at the first half but at the second half only the second arm if $C_{t,1,2} = C_{t,2,2} \leq 0.5$. The dynamic optimal benchmark is $\text{OPT}(T) = 3750$.
\end{enumerate}

For above two examples, the algorithm regrets under different $W$ or $V$ are shown in Figure \ref{fig:measure}.

\begin{figure}[h]
\begin{subfigure}{.5\textwidth}
  \centering
  \includegraphics[width=.95\linewidth]{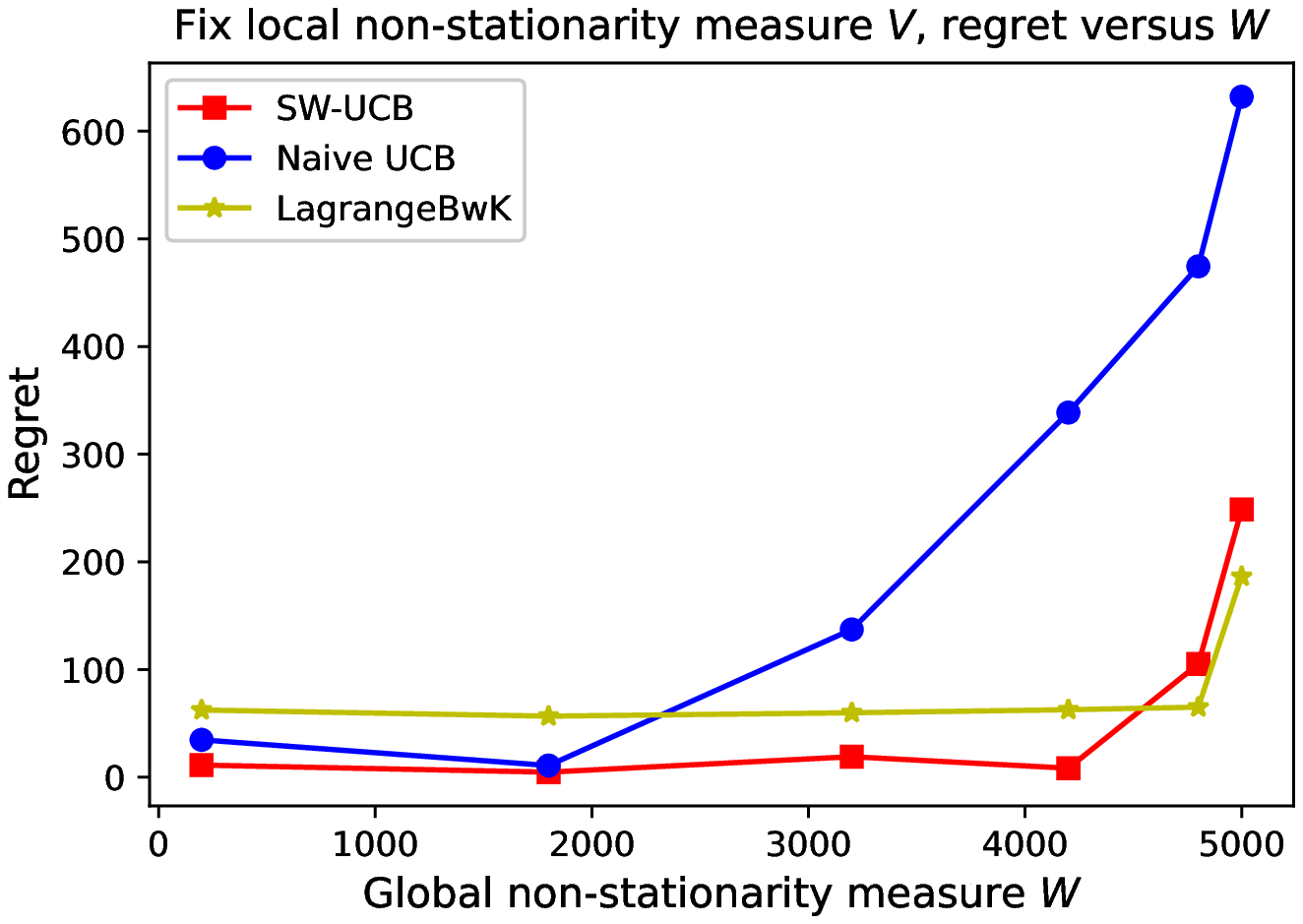}
  \caption{Example 3: Fixed $V$, different $W$'s}
  \label{fig:sfig3}
\end{subfigure}%
\begin{subfigure}{.5\textwidth}
  \centering
  \includegraphics[width=.95\linewidth]{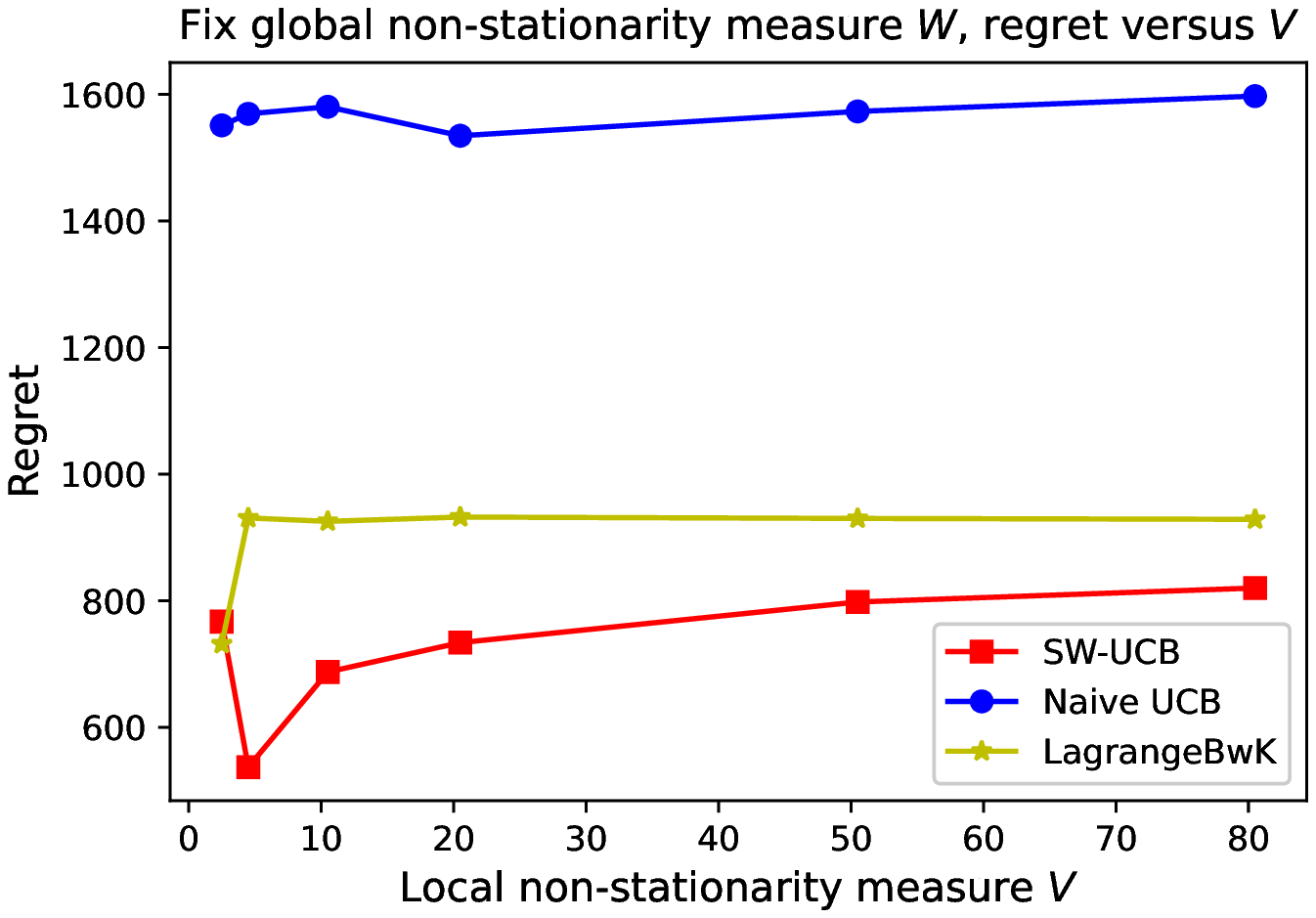}
  \caption{Example 4: Fixed $W$, different $V$'s}
  \label{fig:sfig4}
\end{subfigure}
\caption{Regret versus different $W$'s or $V$'s.}
\label{fig:measure}
\end{figure}

Both SW-UCB and naive UCB divide the resource evenly and assign each part to each time step, which makes their performance depend on the regret induced by this rule (characterized by $W$ in our analysis) to a very large extent (as is shown in Figure \ref{fig:sfig3}). In comparison, LagrangeBwK is not so sensitive to changing $W$ only, partly due to its Hedge way (see \cite{freund1997decision}) to allocate resources. For a relatively small $W$ (which means that the problem is not so \emph{adversarial}), one can expect SW-UCB to outperform LagrangeBwK, while for a very large $W$ the adversarial algorithm could be better. Note that the performance of naive UCB suffers from the abrupt change of the environment (as is shown in Figure \ref{fig:sfig2} and corresponding analysis).

When $W$ is fixed, the regret varies according to the difficulty in identifying the optimal arm distribution, as is shown in Figure \ref{fig:sfig4}. For those environments that do not vary rapidly, SW-UCB achieves a better regret result, while LagrangeBwK is still insensitive to $V$ due to its adversarial nature. When the environment changes rapidly, SW-UCB fails from precisely learning the environment, while still achieving a better result than LagrangeBwK. We note that naive UCB still performs badly in this two-dimensional case.

\section{Proofs of Section \ref{sec_prel} and Section \ref{sec_sl}}
\subsection{Proof of Lemma \ref{lemma:barq_upperbound}}
\label{apd:lemma_barq}
\begin{proof}
We first inspect the null arm (say, the $m$-th arm) where $\mu_{t,m} = 0$ and $C_{t,j,m} = 0$. The global DLP must satisfy that
$$ \mu_{t,m} - \sum_{j=1}^d C_{t,j,m} - (\bm \alpha)_t \leq 0, $$
i.e.
$$ (\bm \alpha)_t \geq 0. $$
The same argument applies to the one-step LP such that $\alpha_t \geq 0$ for all $t=1,...,T$.

Note that the reward is upper bounded by $1$. Hence,
$$ \mathrm{LP}(\{\bm \mu_t\}, \{\bm C_t\}, T) \leq T, $$
$$ \mathrm{LP}(\bm \mu_t, \bm C_t) \leq 1, \quad \forall t = 1,\dots, T. $$

Therefore,
\begin{align*}
    T\| \bm q^*\|_\infty b \leq T \bm b^\top \bm q^* & \leq T\bm b^\top \bm q^* + \sum_{t=1}^T(\bm \alpha^*)_t \\
    & = \mathrm{DLP}(\{\bm \mu_t\}, \{\bm C_t\})\\
    & = \mathrm{LP}(\{\bm \mu_t\}, \{\bm C_t\}, T) \leq T,
\end{align*}
and
\begin{align*}
    \| \bm q_t^*\|_\infty b \leq \bm b^\top \bm q_t^* & \leq \bm b^\top \bm q_t^* + \alpha_t^* \\
    & = \mathrm{DLP}(\bm \mu_t, \bm C_t)\\
    & = \mathrm{LP}(\bm \mu_t, \bm C_t) \leq 1.
\end{align*}

Combining above two inequalities together, we have
$$ \bar{q} b \leq 1. $$
\end{proof}

\subsection{Proof of Proposition \ref{prop:LP}}
\label{apd:prop_LP}
\begin{proof}
The first inequality is straightforward from the fact that the feasible solutions of single-step LP$(\bm \mu_t,\bm C_t)$'s yield a feasible solution for the global LP$(\{\bm \mu_t\}, \{\bm C_t\}, T)$.

For the second inequality, we study the dual problems. By the strong duality of LP, we have
$$ \mathrm{DLP}(\{\bm \mu_t\}, \{\bm C_t\}) = \mathrm{LP}(\{\bm \mu_t\}, \{\bm C_t\}, T), $$
$$ \mathrm{DLP}(\bar{\bm \mu}, \bar{\bm C}) = \mathrm{LP}(\bar{\bm \mu}, \bar{\bm C}).$$
Denote the dual optimal solution w.r.t. $(\bar{\bm \mu}, \bar{\bm C})$ by $(\bar{\bm q}^*, \bar{\alpha}^*)$. Then
$$ \bar{\bm \mu} \leq \bar{\bm C}^\top \bar{\bm q}^* + \bar{\alpha}^* \cdot \mathbf{1}_m $$
implies that
$$ \bm \mu_t \leq \bm C_t^\top \bar{\bm q}^* + \bar{\alpha}^* \cdot \mathbf{1}_m + (\bar{\bm C} - \bm C_t)^\top \bar{\bm q}^* + (\bm \mu_t - \bar{\bm \mu}), \quad \forall t, $$
which induces a feasible solution to the dual program DLP$(\{\bm \mu_t\}, \{\bm C_t\})$, i.e.
$ (\bar{\bm q}^*, \bm \alpha^\prime), $
where
$$ \alpha^\prime_t \coloneqq \bar{\alpha}^* + \| (\bar{\bm C} - \bm C_t)^\top \bar{\bm q}^* + (\bm \mu_t - \bar{\bm \mu}) \|_{\infty}.  $$
Hence,
\begin{align*}
    \mathrm{DLP}(\{\bm \mu_t\}, \{\bm C_t\}) - T\cdot \mathrm{DLP}(\bar{\bm \mu}, \bar{\bm C})
    & \leq \sum_{t=1}^T \| (\bar{\bm C} - \bm C_t)^\top \|_{\infty} \|\bar{\bm q}^*\|_{\infty} + \sum_{t=1}^T \|\bm \mu_t - \bar{\bm \mu}\|_{\infty} \\
    & = \sum_{t=1}^T \|\bar{\bm C} - \bm C_t\|_{1} \|\bar{\bm q}^*\|_{\infty} + \sum_{t=1}^T \|\bm \mu_t - \bar{\bm \mu}\|_{\infty} \\
    &\leq \bar{q} W_2 + W_1.
\end{align*}

For the last inequality, similar duality arguments can be made with respect to $T = 1$. Taking a summation, we yield the final inequality as desired.
\end{proof}

\subsection{Proofs of Lemma \ref{lemma:confidence_bound} and Lemma \ref{lemma:reward_consumption}}
\label{apd:lemma_r_c}
\begin{lemma}[Azuma-Hoeffding's inequality]
\label{lemma:concentration}
Consider a random variable with distribution supported on $[0,1]$. Denote its expectation as $z$. Let $\bar{Z}$ be the average of $N$ independent samples from this distribution. Then, $\forall \delta > 0$, the following inequality holds with probability at least $1 - \delta$,
$$ |\bar{Z} - z| \leq \sqrt{\frac1{2N} \log\left(\frac{2}{\delta}\right)}. $$
More generally, this result holds if $Z_1, \dots, Z_N \in [0,1]$ are random variables, $ \bar{Z} = \frac{1}{N}\sum_{n=1}^N Z_t$, and $ z = \frac{1}{N}\sum_{n=1}^N \mathbb{E}[Z_n|Z_1, \dots, Z_{n-1}]$.
\end{lemma}

Next, we present a general bound for the normalized empirical mean of the sliding-window estimator:
\begin{lemma}
\label{lemma:normalized}
For any window size $w$, define the normalized empirical average within window size $w$ of some $Z_{t,i} \in [0,1]$ with mean $z_{t,i}$ for each arm $i$ at time step $t$ as
$$ \hat{Z}_{t,i}^{(w)} \coloneqq \frac{\sum_{s = 1\vee(t-w)}^{t-1} Z_{t} \cdot \mathbbm{1}\{i_s = i\}}{n_{t,i}^{(w)} + 1}, $$
where $n_{t, i}^{(w)} \coloneqq \sum_{s=1\vee(t-w)}^{t-1} \mathbbm{1}\{i_s = i\}$ is the number of plays of arm $i$ before time step $t$ within $w$ steps. Then for small $\delta$ such that $\log\left(\frac{2}\delta\right) > 2$, the following inequality holds with probability at least $1 - \delta$,
$$ |\hat{Z}_{t,i}^{(w)} - z_{t,i}| \leq \sqrt{\frac{2}{n_{t,i}^{(w)} + 1} \log\left(\frac{2}{\delta}\right)} + \sum_{s=1\vee(t-w)}^{t-1} |z_{s,i} - z_{s+1,i}|. $$
\end{lemma}
\begin{proof}
The result follows from applying Lemma \ref{lemma:concentration} to the empirical mean. For the case when $n_{t,i}^{(w)} = 0$, the result automatically holds. When $n_{t,i}^{(w)} \ge 1$,
\begin{align*}
    |\hat{Z}_{t,i}^{(w)} - z_{t,i}| 
    & \leq \frac{n_{t,i}^{(w)}}{n_{t,i}^{(w)} + 1} \left|\hat{Z}_{t,i}^{(w)} - \frac{\sum_{s=1\vee(t-w)}^{t-1} z_{s,i} \cdot \mathbbm{1}\{i_s = i\}}{n_{t,i}^{(w)}} \right| +  \sum_{s=1\vee(t-w)}^{t-1} \frac{|z_{s,i} - z_{t,i}| \cdot \mathbbm{1}\{i_s = i\}}{n_{t,i}^{(w)} + 1} \\
    & \phantom{\leq} + \frac{z_{t,i}}{n_{t,i}^{(w)} + 1}\\
    & \leq \frac{n_{t,i}^{(w)}}{n_{t,i}^{(w)} + 1} \sqrt{\frac{1}{2 n_{t,i}^{(w)}} \log \left(\frac{2}\delta\right)} +  \sum_{s=1\vee(t-w)}^{t-1} \sum_{p=s}^{t-1} \ \frac{|z_{p,i} - z_{p+1,i}| \cdot \mathbbm{1}\{i_s = i\}}{n_{t,i}^{(w)} + 1} \\
    &\phantom{\leq}+ \sqrt{\frac{1}{2 (n_{t,i}^{(w)} + 1)} \log \left(\frac{2}\delta\right)}\\
    & \leq \sqrt{\frac{2}{n_{t,i}^{(w)} + 1} \log\left(\frac{2}\delta\right)} +  \sum_{p=1\vee(t-w)}^{t-1} \frac{\sum_{s=1\vee(t-w)}^{p}\mathbbm{1}\{i_s = i\}}{n_{t,i}^{(w)} + 1}|z_{p,i} - z_{p+1,i}|\\
    & \leq \sqrt{\frac{2}{n_{t,i}^{(w)} + 1} \log\left(\frac{2}\delta\right)} +  \sum_{p=1\vee(t-w)}^{t-1} |z_{p,i} - z_{p+1,i}|,
\end{align*}
where the first inequality comes from definition of $\hat{Z}_{t,i}^{(w)}$ and triangular inequality, the second inequality comes from Lemma \ref{lemma:concentration}, triangular inequality, and the fact that $z_{t,i} \leq 1 \leq \sqrt{\frac{\log(2/ \delta)}{2 (n_{t,i}^{(w)} + 1)}}$, the third inequality comes from the fact that $n_{t,i}^{(w)} \leq n_{t,i}^{(w)} + 1$ and rearranging the sum of $p$ and $s$, and the last inequality comes from the fact that $\sum_{s=1\vee(t-w)}^p \mathbbm{1}\{i_s = i\} \leq \sum_{s=1\vee(t-w)}^t \mathbbm{1}\{i_s = i\} = n_{t,i}^{(w)} \leq n_{t,i}^{(w)} + 1$.
\end{proof}

Using Lemma \ref{lemma:normalized}, we can easily derive the proof of Lemma \ref{lemma:confidence_bound}:
\begin{proof}
Replacing $z_{t,i}$ by $\mu_{t,i}$ and $C_{t,j,i}$, $\hat{Z}_{t,i}^{(w)}$ by $\hat{\mu}_{t,i}^{(w_1)}$ and $\hat{C}_{t,j,i}^{(w_2)}$ accordingly in Lemma \ref{lemma:normalized} yields the final result.
\end{proof}

We now start the main proof of Lemma \ref{lemma:reward_consumption}.
\begin{proof}
We will prove the result for reward first. By Lemma \ref{lemma:normalized}, with probability at least $1 - \frac{1}{6T}$, $\forall t \leq \min\{\tau, T\}$,
\begin{align*}
    & \phantom{\leq} \left|\sum_{s=1}^t (\mu_{s,i_s} - \mathrm{UCB}_{s,i_s}(\bm \mu_s))\right| \\
    & \leq 2\sum_{s=1}^t \sqrt{\frac{2}{n_{s,i_s}^{(w_1)} + 1} \log (12mT^3)} + \sum_{s=1}^t \sum_{p=1\vee(t-w_1)}^{s-1} \|\bm \mu_p - \bm \mu_{p+1}\|_{\infty} \\
    & \leq 2 \sum_{s=1}^t \sqrt{\frac{2\log(12mT^3)}{\sum_{p=1\vee(t-w_1)}^{s-1}\mathbbm{1}\{i_p = i_s\} + 1}} + w_1 V_1 \\
    & \leq 2 \sum_{k=0}^{\lceil t/w_1 \rceil - 1}\sum_{s=kw_1+1}^{(k+1) w_1} \frac{\sqrt{2\log(12mT^3)}}{\sqrt{\sum_{p=kw_1+1}^{s-1}\mathbbm{1}\{i_p = i_s\} + 1}} + w_1 V_1\\
    & = 2 \sum_{k=0}^{\lceil t/w_1 \rceil -1} \sum_{i=1}^m \sum_{n=1}^{N_{k,i}^{(w_1)}} \frac{\sqrt{2\log(12mT^3)}}{\sqrt{n}} + w_1 V_1,
\end{align*}
where $N_{k,i}^{(w_1)} \coloneqq \sum_{p=kw_1+1}^{(k+1)w_1} \mathbbm{1}\{i_p = i\}, \quad \sum_{i=1}^m N_{k,i}^{(w_1)} = w_1$.
Here the first inequality comes from Lemma \ref{lemma:normalized}, the second inequality comes from the fact that $\sum_{s=1}^t \sum_{p = 1\vee(t-w_1)}^{s-1} \|\bm \mu_p - \bm \mu_{p+1}\|_\infty \leq w_1 V_1$, the third comes from cutting time steps into $\lceil t/w_1\rceil$ periods, and the last comes from rearranging the sum with respect to $i$.
Then we have
\begin{align}
\label{eq:reward1}
    & \phantom{\leq} \left|\sum_{s=1}^t (\mu_{s,i_s} - \mathrm{UCB}_{s,i_s}(\bm \mu_s))\right| \nonumber \\
    & \leq 2 \sum_{k=0}^{\lceil t/w_1 \rceil -1} \sum_{i=1}^m 2\sqrt{2\log(12mT^3)}\sqrt{N_{k,i}^{(w_1)}} + w_1 V_1 \nonumber\\
    & \leq 2 \sum_{k=0}^{\lceil t/w_1 \rceil -1} 2\sqrt{2\log(12mT^3) m w_1} + w_1 V_1 \nonumber\\
    & \leq 8 \sqrt{2\log(12mT^3) m} \cdot \frac{T}{\sqrt{w_1}} + w_1 V_1,
\end{align}
where the first inequality comes from the fact that $\sum_{n=1}^N \frac{1}{\sqrt{n}} \leq 2 \sqrt{N}$, the second inequality comes from Cauchy-Schwarz inequality, and the last comes from the fact that $\lceil \frac{t}{w_1} \rceil \leq \frac{2T}{w}$.

Furthermore, applying Lemma \ref{lemma:concentration} to $r_s$ and $\mu_{s,i_s}$, we have that with probability at least $1 - \frac{1}{6mT^2}$, $\forall t \leq \min\{\tau,T\}$,
\begin{equation}
\label{eq:reward2}
    \left|\sum_{s=1}^t (r_s - \mu_{s,i_s})\right| \leq \sqrt{2T \log(12mT^3)}.
\end{equation}

Then we apply Lemma \ref{lemma:concentration} to $\mathrm{UCB}_s(\bm \mu_s)^\top \bm x_s^*$ and $\mathrm{UCB}_{s,i_s}(\bm \mu_s)$ and note that $\mathrm{UCB}_{s,i}(\bm \mu_s) \in [0,1 + \sqrt{2 \log (12mT^3)}]$. It yields that with probability at least $1-\frac{1}{6mT^2}$, $\forall t \leq \min\{\tau, T\}$,
$$ \left|\sum_{s=1}^t (\mathrm{UCB}_s(\bm \mu_s)^\top \bm x_s^* - \mathrm{UCB}_{s,i_s}(\bm \mu_s))\right| \leq (1 + \sqrt{2 \log (12mT^3)}) \sqrt{2T \log (12mT^3)}. $$
Combining all these inequalities (\ref{eq:reward1}), (\ref{eq:reward2}), (\ref{eq:reward3}) together, we have with probability at least $1-\frac{1}{2T}$, $\forall t \leq \min\{\tau, T\}$,
\begin{align}
\label{eq:reward3}
    &\phantom{\leq}\left|\sum_{s=1}^t (r_s - \mathrm{UCB}_s(\bm \mu_s)^\top \bm x_s^* )\right| \nonumber \\
    & \leq (2+ \sqrt{2 \log (12mT^3)}) \sqrt{2T \log (12mT^3)} + 8 \sqrt{2\log(12mT^3) m} \cdot \frac{T}{\sqrt{w_1}} + w_1 V_1 \nonumber\\
    & \leq 2 \sqrt{2 \log (12mT^3)}\cdot \sqrt{2T \log (12mT^3)} +8 \sqrt{2\log(12mT^3) m} \cdot \frac{T}{\sqrt{w_1}} + w_1 V_1 \nonumber\\
    & = 4 \sqrt{T} \log (12mT^3) + 8 \sqrt{2\log(12mT^3) m} \cdot \frac{T}{\sqrt{w_1}} + w_1 V_1,
\end{align}
where the first inequality comes from inequalities (\ref{eq:reward1}), (\ref{eq:reward2}), (\ref{eq:reward3}), and the second inequality comes from the fact that $\log(12mT^3) \geq 2$.

As for the resource consumption, by Lemma \ref{lemma:normalized}, with probability at least $1 - \frac{1}{6T}$, $\forall t \leq \min\{\tau,T\}$
\begin{align*}
    & \phantom{\leq} \left|\sum_{s=1}^t (C_{s,j,i_s} - \mathrm{LCB}_{s,i_s}(\bm C_{s,j}))\right| \\
    & \leq 2 \sum_{s=1}^t \sqrt{\frac{2}{n_{s,i_s}^{(w_2)} + 1} \log(12 md T^3)}+ \sum_{s=1}^t \sum_{p=1\vee(t-w_2)}^{s-1} \|\bm C_{p,j} - \bm C_{p+1,j}\|_{\infty} \\
    & \leq 2 \sum_{s=1}^t \sqrt{\frac{2\log(12 md T^3)}{\sum_{p=1\vee(t-w_2)}^{s-1}\mathbbm{1}\{i_p = i_s\} + 1}} + w_2 V_2 \\
    & \leq 2 \sum_{k=0}^{\lceil t/w_2 \rceil - 1}\sum_{s=kw_2+1}^{(k+1) w_2} \frac{\sqrt{2\log(12 md T^3)}}{\sqrt{\sum_{p=kw_2+1}^{s-1}\mathbbm{1}\{i_p = i_s\} + 1}} + w_2 V_2\\
    & = 2 \sum_{k=0}^{\lceil t/w_2 \rceil -1} \sum_{i=1}^m \sum_{n=1}^{N_{k,i}^{(w_{2})}}\frac{\sqrt{2\log(12 md T^3)}}{\sqrt{n}} + w_2 V_2,
\end{align*}
where $N_{k,i}^{(w_2)} \coloneqq \sum_{p=kw_2+1}^{(k+1)w_2} \mathbbm{1}\{i_p = i\}, \quad \sum_{i=1}^m N_{k,i}^{(w_2)} = w_2$. Here the first inequality comes from Lemma \ref{lemma:normalized}, the second inequality comes from the fact that $\sum_{s=1}^t \sum_{p = 1\vee(t-w_1)}^{s-1} \|\bm C_{p,j} - \bm C_{p+1,j}\|_\infty \leq w_2 V_2$, the third comes from cutting time steps into $\lceil t/w_1\rceil$ periods, and the last comes from rearranging the sum with respect to $i$.
Then we have
\begin{align}
\label{eq:consumption1}
    & \phantom{\leq} \left|\sum_{s=1}^t (C_{s,j,i_s} - \mathrm{LCB}_{s,i_s}(\bm C_{s,j}))\right| \nonumber \\
    & \leq 2 \sum_{k=0}^{\lceil t/w_2 \rceil -1} \sum_{i=1}^m 2\sqrt{2\log(12 md T^3)}\sqrt{N_{k,i}^{(w_2)}} + w_2 V_2 \nonumber \\
    & \leq 2 \sum_{k=0}^{\lceil t/w_2 \rceil -1} 2\sqrt{2\log(12 md T^3) m w_2} + w_2 V_2 \nonumber \\
    & \leq 8 \sqrt{2\log(12 md T^3) m} \cdot \frac{T}{\sqrt{w_2}} + w_2 V_2,
\end{align}
where the first inequality comes from the fact that $\sum_{n=1}^N \frac{1}{\sqrt{n}} \leq 2 \sqrt{N}$, the second inequality comes from Cauchy-Schwarz inequality, and the last comes from the fact that $\lceil \frac{t}{w_1} \rceil \leq \frac{2T}{w}$.

Similarly, we apply Lemma \ref{lemma:concentration} to $\mathrm{LCB}_s(\bm C_{s,j})^\top \bm x_s^*$ and $\mathrm{LCB}_{s,i_s}(\bm C_{s,j})$ and note that $\mathrm{LCB}_{s,i}(\bm C_{s,j}) \in [-\sqrt{2\log(12 md T^3)}, 1]$. It yields that with probability at least $1 - \frac{1}{6mdT^2}$, $\forall t \leq \min\{\tau,T\}$,
\begin{equation}
\label{eq:consumption2}
    \left|\sum_{s=1}^t (\mathrm{LCB}_s(\bm C_{s,j})^\top \bm x_s^* - \mathrm{LCB}_{s,i_s}(\bm C_{s,j}))\right| \leq (1+ \sqrt{2 \log(12 md T^3)}) \sqrt{2 T \log(12 md T^3)}.
\end{equation}
Furthermore, applying Lemma \ref{lemma:concentration} to $c_{s,j}$ and $C_{s,j,i_s}$ induces that with probability at least $1 - \frac{1}{6md T^2}$, $\forall t \leq \min\{\tau,T\}$,
\begin{equation}
\label{eq:consumption3}
    \left|\sum_{s=1}^t (c_{s,j} - C_{s,j,i_s})\right| \leq \sqrt{2T \log(6md T^3)}.
\end{equation}
Combining all these inequalities (\ref{eq:consumption1}), (\ref{eq:consumption2}), (\ref{eq:consumption3}) together, we have with probability at least $1 - \frac{1}{2T}$, $\forall t \leq \min\{\tau,T\}$,
\begin{align*}
    &\phantom{\leq}\left|\sum_{s=1}^t (c_{s,j} - \mathrm{LCB}_s(\bm C_{s,j})^\top \bm x_s^* )\right| \\
    & \leq (2+\sqrt{2 \log(6 md T^3)}) \sqrt{2T \log(6md T^3)} + 8 \sqrt{2\log(12 md T^3) m} \cdot \frac{T}{\sqrt{w_2}} + w_2 V_2 \\
    & \leq 4\sqrt{T}\log(6md T^3) + 8 \sqrt{2\log(12 md T^3) m} \cdot \frac{T}{\sqrt{w_2}} + w_2 V_2.
\end{align*}
\end{proof}

\subsection{Proof of Corollary \ref{coro:terminate}}
\label{apd:coro_terminate}
\begin{proof}
Without loss of generality, we only analyze the case that $\tau \leq T$, i.e. the resource constraint is violated before time step $T$.
At termination time $\tau$, we have
$$ \sum_{t=1}^\tau c_{t,j} > b T $$
for some $j \leq d$.

From the fact that $\bm x_t^*$ is a feasible solution to $\mathrm{LP}(\mathrm{UCB}_t(\bm \mu_t), \mathrm{LCB}_t(\bm C_t))$, we have
$$ \sum_{t=1}^{\tau} \mathrm{LCB}_t(\bm C_t) \bm x_t^* \leq b \tau. $$
Combining that inequality with Lemma \ref{lemma:reward_consumption}, we have with probability at least $1 - \frac{1}{2T}$
$$ \sum_{t=1}^\tau c_{t,j} \leq b \tau + 4 \sqrt{T} \log(12mdT^3) + 14 m^{\frac{1}{3}} V_2^{\frac{1}{3}} T^{\frac{2}{3}} \log^{\frac{1}{3}} (12 mdT^3) + 8\sqrt{2mT} \sqrt{\log(mdT^3)}. $$
Therefore we have
$$ b \tau + 4 \sqrt{T} \log(12mdT^3) + 14 m^{\frac{1}{3}} V_2^{\frac{1}{3}} T^{\frac{2}{3}} \log^{\frac{1}{3}} (12 mdT^3) + 8\sqrt{2mT}\sqrt{\log(mdT^3)} > b T, $$
which yields the final result.
\end{proof}

\subsection{Proof of Theorem \ref{thm:upper}}
\label{apd:thm_upper}
\begin{proof}
From Lemma \ref{lemma:confidence_bound}, we know that with probability at least $1 - \frac{1}{3T}$,
\begin{align*}
  \phantom{=}\sum_{t=1}^{\tau - 1} \mathrm{UCB}_t(\bm \mu_t)^\top \bm x_t 
    & = \sum_{t=1}^{\tau - 1} \mathrm{LP}(\mathrm{UCB}_t(\bm \mu_t), \mathrm{LCB}_t(\bm C_t))\\
    & \geq \sum_{t=1}^{\tau - 1} \mathrm{LP}\left(\bm \mu_t - \mathbf{1}\cdot \sum_{s=1\vee(t-w_1)}^{t-1} \|\bm \mu_s - \bm \mu_{s+1}\|_{\infty}, \ \bm C_t + \sum_{j=1}^d \bm E_j \cdot \sum_{s=1\vee(t-w_{2,j})}^{t-1} \|\bm C_{s,j} - \bm C_{s+1,j}\|_{\infty}\right)\\
    & \geq \sum_{t=1}^{\tau - 1} \mathrm{LP}(\bm \mu_t, \bm C_t) - \sum_{t=1}^{\tau - 1} \sum_{s=1\vee(t-w_1)}^{t-1} \|\bm \mu_s - \bm \mu_{s+1}\|_{\infty} - \bar{q} \sum_{t=1}^{\tau - 1} \sum_{j=1}^d \sum_{s=1\vee(t-w_{2,j})}^{t-1} \|\bm C_{s,j} - \bm C_{s+1,j} \|_{\infty}\\
    & \geq (\tau - 1)\mathrm{LP}(\bar{\bm \mu}, \bar{\bm C}) - (W_1 + \bar{q} W_2) - (w_1 V_1 + \bar{q} d w_2 V_2),
\end{align*}
where $\bm E_j$ is the matrix that is $\mathbf{1}^\top$ at the $j$-th row while other components all zeros.
Here the first inequality comes from Lemma \ref{lemma:confidence_bound}, the second inequality comes from the proof of Proposition \ref{prop:LP}, and the last inequality comes from applying Proposition \ref{prop:LP} to $\sum_{t=1}^{\tau-1}\mathrm{LP}(\bm \mu_t, \bm C_t)$.

For Lemma \ref{lemma:reward_consumption},
if we select $w_1 = \min\left\{\lceil m ^{\frac{1}{3}} V_1^{-\frac{2}{3}} T^{\frac{2}{3}} \log^{\frac{1}{3}} (12mT^3) \rceil, \  T\right\}$, we have $\forall t \leq \min\{\tau,T\}$
$$ \left|\sum_{s=1}^t (r_s - \mathrm{UCB}_s(\bm \mu_s)^\top \bm x_s^* )\right| \leq 4\sqrt{T} \log (12mT^3) + 14 m ^{\frac{1}{3}} V_1^{\frac{1}{3}} T^{\frac{2}{3}} \log^{\frac{1}{3}}(12mT^3) + 8 \sqrt{2 m T} \sqrt{\log(12mT^3)}. $$

Therefore, by Lemma \ref{lemma:reward_consumption} and Corollary \ref{coro:terminate}, with probability at least $1 - \frac{1}{T}$,
\begin{align*}
    &\phantom{\leq}\mathrm{LP}(\{\bm \mu_t\}, \{\bm C_t\}, T) - \sum_{t=1}^{\tau - 1} r_t\\
    & = (\mathrm{LP}(\{\bm \mu_t\}, \{\bm C_t\}, T) - \sum_{t=1}^{\tau - 1} \mathrm{UCB}_t(\bm \mu_t)^\top \bm x_t) + (\sum_{t=1}^{\tau - 1} \mathrm{UCB}_t(\bm \mu_t)^\top \bm x_t - \sum_{t=1}^{\tau - 1}r_t) \\
    & \leq (4\sqrt{T}\log(12 md T^3) + 14 m ^{\frac{1}{3}} V_2^{\frac{1}{3}} T^{\frac{2}{3}}  \log^{\frac{1}{3}}(12mdT^3) + 8 \sqrt{2mT}\sqrt{\log(12mdT^3)} + 1) \cdot \frac{\mathrm{LP}(\bar{\bm \mu}, \bar{\bm C})}{b}\\
    &\phantom{=} + 4\sqrt{T}\log(12 m T^3) + 14 m ^{\frac{1}{3}} V_1^{\frac{1}{3}} T^{\frac{2}{3}}  \log^{\frac{1}{3}}(12mT^3) + 8\sqrt{2mT} \sqrt{\log(12mT^3)} \\
    & \phantom{=}+ 2(W_1 + \bar{q} W_2) + w_1 V_1 + \bar{q}d w_2 V_2\\
    & = O\left(\frac{1}{b} \sqrt{mT} \log(md T^3)
    + m ^{\frac{1}{3}} V_1^{\frac{1}{3}} T^{\frac{2}{3}} \log^{\frac{1}{3}}(mT^3) + \frac{1}{b}\cdot m ^{\frac{1}{3}} d V_2^{\frac{1}{3}} T^{\frac{2}{3}}  \log^{\frac{1}{3}}(mdT^3)+ W_1 + \bar{q}W_2\right),
\end{align*}
where we utilize the fact that $\bar{q}d \leq \frac{1}{b}\cdot d$ by Lemma \ref{lemma:barq_upperbound} at the last equality.

Note that $\mathrm{OPT}$ is of linear $T$ (in fact, $\mathrm{OPT} \leq \mathrm{LP}(\{\bm \mu_t\}, \{\bm C_t\}, T) \leq T$), which transforms the high probability bound into the expectation bound.
\end{proof}

\subsection{Proof of Theorem \ref{thm:lower}}
\label{apd:thm_lower}
\begin{proof}
The first lower bound follows directly from \cite{besbes2014stochastic}. For here, we provide a brief description for completeness. The time horizon is divided into $\lceil \frac{T}{H} \rceil$ periods, where each is of length $H$ except possibly the last one ($H$ to be specified). For each period, the nature selects an arm to be optimal uniformly randomly and independently, which is of mean reward $r^* = \frac{1}{2} + \Delta$, where the other arms are all of $r = \frac{1}{2}$. Then from the information-theoretic arguments of the standard multi-armed bandits problem, if we select $\Delta = \Theta(\sqrt{\frac{m}{H}})$, we must suffer an expected regret of $\Omega(\sqrt{mH})$ at each period for any policy. Here we assume that $H$ is large enough such that $\Delta = \Theta(\sqrt{\frac{m}{H}}) \leq \frac{1}{2}$. Therefore, the total regret is of $\Omega(\sqrt{\frac{m}{H}} T)$, where the local non-stationarity budget $V_1 = \Theta(\Delta \frac{T}{H}) = \Theta(m^{\frac{1}{2}}H^{-\frac{3}{2}} T)$.  Note that the example yields a regret of $\Omega(m^{\frac{1}{3}} V_1^{\frac{1}{3}} T^{\frac{2}{3}})$ by selecting $H = \Theta(m^{\frac{1}{3}} V_1^{-\frac{2}{3}} T^{\frac{2}{3}})$. 

For the second lower bound, we can establish based on some modification of the first example. We now assume that each arm is of deterministic reward $r = 1$ and there is only one type of resource. The only difference among the arms is on the resource consumption. To avoid the complication of stopping time, we split the time horizon $T$ into two halves in a way such that the extra consumption of resource at the first half can be deterministically transformed into the reward loss due to limited resource at the second half. Specifically, for the second half, every arm generates a deterministic reward of $r$ and a deterministic consumption $b$. For the first half, the nature divides it in a similar way as the first lower bound example and the goal here is to generate an inevitably excessive resource consumption compared to dynamic optimal policy. There are $\lceil \frac{T}{2H} \rceil$ periods which are of length $H$. Among these periods, the nature uniformly and independently chooses an arm to be optimal. We assume that the optimal arm is of mean consumption $c^* = b$ while the others are of mean consumption $c = b + \Delta = b + \Theta(\sqrt{\frac{m}{H}})$. We can without loss of generality consider only those cases where the resource is not all consumed at the first half (otherwise, the rewards collected will be at least $\frac{1}{2} \mathrm{OPT}$ smaller than $\mathrm{OPT}$, where the conclusion is automatically fulfilled). By similar information-theoretical arguments, any policy must suffer an expected additional consumption of $\Omega(\sqrt{\frac{m}{H}} T)$ compared to the dynamic optimal policy at the first half, which in turn yields an expected regret of $\Omega(\frac{1}{b} \sqrt{\frac{m}{H}} T)$. By selecting $H = \Theta(m^{\frac{1}{3}} V_2^{-\frac{2}{3}} T^{\frac{2}{3}})$, we construct an example of regret $\Omega(m^{\frac{1}{3}} V_2^{\frac{1}{3}} T^{\frac{2}{3}} \cdot \frac{1}{b})$. 

The third lower bound example is constructed based on the motivating example in Section \ref{motivate_eg}. We can consider a one-armed bandit problem and divide the time horizon into two halves. There is only one resource type and the total available resource is $B=bT$ with $b < \frac12$. At the first half, the arm generates a deterministic reward $r$ and consumes resource $2 b$. At the second half, the nature randomly chooses between the following two cases: the situation either becomes better with reward $r+\Delta_1$ and consumption $2 b-\Delta_2$ or worse with reward $r-\Delta_1$ and consumption $2 b+\Delta_2$. For the first situation, the optimal policy is to reserve the resource as much as possible for the second half whilst for the second situation it is optimal to consume all the resource at the first half. One can choose $r, b$ such that $\frac{r}{b} = \Theta(\bar{q})$. With a similar argument as in Section \ref{motivate_eg}, the algorithm will suffer a regret of 
$$\Theta(T(\Delta_1 + \frac{r}{b} \Delta_2)) = \Theta(W_1 + \bar{q}W_2)$$ for at least one of the two situations. 
\end{proof}

\section{Discussions on the Benchmarks and Tightening the Measures}

\subsection{Benchmarks used in BwK literature}

In the subsection, we provide a thorough discussion on the four benchmarks for the BwK problem. Specifically, our dynamic benchmark is the strongest one in comparison with others, and we are the first one to analyze against this benchmark in a non-stochastic (non-i.i.d.) environment.

$\text{OPT}_{\text{DP}}$: It is defined by an optimal algorithm that utilizes the knowledge of the true underlying distributions and maximizes the expected cumulative reward $E[\sum_{t=1}^T r_t]$ subject to the knapsack constraints. This is called as the dynamic optimal benchmark and it is used in the stochastic BwK literature, for both problem-independent bounds \citep{badanidiyuru2013bandits, agrawal2014bandits}, and problem-dependent bounds \citep{sankararaman2021bandits, li2021symmetry}.

$\text{OPT}_{\text{FD}}$: It is called as the fixed distribution benchmark considered in the adversarial BwK problem \citep{immorlica2019adversarial}. It is also defined based on an algorithm that utilizes the knowledge of the true underlying distributions and maximizes the expected cumulative reward. But importantly, the algorithm is required to play the arms following a fixed (static) distribution throughout the horizon. As mentioned earlier, the dynamic optimal benchmark is more relevant for the practical applications of BwK than this fixed distribution benchmark.

$\text{OPT}_{\text{LP-Dynamic}}:$ It is defined by the optimal value of the following linear program (LP):
\begin{align*}
  \text{OPT}_{\text{LP-Dynamic}} \coloneqq \mathrm{LP}\left(\{ \mu_t\}, \{ C_t\}, T\right) \ \coloneqq \ & \max_{ x_1,\dots, x_T} \ \sum_{t=1}^T  \mu_t^\top  x_t \\
    & {s.t. } \sum_{t=1}^T  C_t  x_t \le  B,\quad  x_t \in \Delta_m,\ t = 1,\dots, T,
\end{align*}
and this is the benchmark used in our paper. The LP's inputs ${\mu}_{t}$ and ${C}_t$ are the vector/matrix of the expected reward and resource consumption at time $t$. The decision variables ${x}_t$ stay within the standard simplex $\Delta_m$ and it can be interpreted as a random arm play distribution for time $t$. The benchmark is also known as deterministic, fluid, or prophet benchmark. It is commonly adopted in the literature for its tractability in analysis than the dynamic benchmark $\text{OPT}_{\text{DP}}$. 

$\text{OPT}_{\text{LP-Static}}:$ It is defined by requiring ${x}_1={x}_2=...={x}_T$ in the above LP. This is apparently a weaker benchmark, and it can be viewed as a deterministic upper bound of the $\text{OPT}_{\text{DP}}$. 

The following inequality holds
$$\text{OPT}_{\text{FD}} \overset{(1)}{\le} \text{OPT}_{\text{DP}}  \overset{(2)}{\le }{\text{OPT}}_{\text{LP-Dynamic}} $$
$$\text{OPT}_{\text{FD}}\overset{(3)}{\le} \text{OPT}_{\text{LP-Static}} \overset{(4)}{\le} \text{OPT}_{\text{LP-Dynamic}}.$$

Here (1) and (4) are evident because of the extra requirement of fixed distribution (for (1)) and extra constraint of $x_1=...=x_T$ (for (4)). For (2) and (3), they can be proved by a convexity argument with Jensen's inequality on the realized sample path and the expectation. 

We make the following two remarks:

First, when the underlying environment is stochastic (stationary), the expected reward and resource consumption, $\mu_1=...=\mu_T$ and $C_1=...=C_T.$ The optimal solution of the LP in defining $\text{OPT}_{\text{LP-Dynamic}}$ automatically satisfies $x_1^* =...=x_T^*.$ So, for a stochastic environment
$$\text{OPT}_{\text{LP-Static}}=\text{OPT}_{\text{LP-Dynamic}}.$$
The existing literature on stochastic BwK (such as \cite{badanidiyuru2013bandits, agrawal2014bandits}) uses this equivalent benchmark to analyze the upper bound of the algorithm regret. 

Second, any of these benchmark definition will not restrict it to distributions that do not exhaust budget until T rounds. The LP benchmarks will always upper bound the benchmarks of $\text{OPT}_{\text{FD}}$ and $\text{OPT}_{\text{DP}}.$ The LP benchmarks allow early exhaustion as well, because the presence of the null arm allows an play that consume zero resource. This is also reflected by the inequality in the LP's constraints, otherwise if early exhaustion is not allowed, it should be equality in the LP's constraints. 

Furthermore, we allow $\mathcal{P}_t$ to be point-mass distributions and allow it to be chosen adversarially. So our non-stationary setting does not conflict with the adversarial setting and it indeed recovers the adversarial BwK as one end of the spectrum. The non-stationarity measures aim to relate the best-achievable algorithm performance with the intensity of adversity of the underlying environment.

\subsection{Tightening the global measures $W_1$ and $W_2$}

\label{sec_better_W}

In this section, we discuss how to improve the global non-stationarity measures $W_1$ and $W_2$. First, we revise the definitions of $W_1$ and $W_2$, and write the non-stationarity measures as functions:
$$ W_1(\bm \mu) \coloneqq \sum_{t=1}^T \|\bm \mu_t - \bm \mu \|_{\infty}, \quad W_2(\bm C) \coloneqq \sum_{t=1}^T \| \bm C_t - \bm C \|_1. $$
Specifically, we note that $W_1(\bar{\bm \mu})=W_1$ and $W_2(\bar{\bm C})=W_2$. 

We note from the proof of Proposition \ref{prop:LP} that the following inequalities hold 
$$ \sum_{t=1}^T \mathrm{LP}(\bm \mu_t, \bm C_t) \leq \mathrm{LP}(\{\bm \mu_t\}, \{\bm C_t\}, T) \leq T \cdot \mathrm{LP}(\tilde{\bm\mu}, \tilde{\bm C}) +W_1(\tilde{\bm{\mu}})+ \bar{q}W_2(\tilde{\bm{C}}) \leq \sum_{t=1}^T \mathrm{LP}(\bm \mu_t, \bm C_t) + 2(W_1(\tilde{\bm{\mu}})+ \bar{q} W_2(\tilde{\bm{C}}))$$
for any $\tilde{\bm\mu}$ and $\tilde{\bm C}$ if the optimal dual solution $\tilde{\bm{q}}^*$ of $\mathrm{LP}(\tilde{\bm\mu}, \tilde{\bm C})$ satisfies $\|\tilde{\bm{q}}^*\|_\infty \le \bar{q}$. 

That is, if $\|\tilde{\bm{q}}^*\|_\infty \le \bar{q}$ is satisfied, the terms of $W_1$ and $W_2$ in Proposition \ref{prop:LP} and Theorem \ref{thm:upper} can be replaced by $W_1(\tilde{\bm{\mu}})$ and $ W_2(\tilde{\bm{C}})$, respectively.

Therefore, a natural idea is to refine the two non-stationarity measures based on a combination of $\bm{\mu}$ and $\bm{C}$ that minimize the two functions, i.e., to define,
\begin{align*}
    W_1^{\min} \ \coloneqq \ \min_{\bm \mu} \ \sum_{t=1}^T \|\bm \mu_t - \bm \mu \|_{\infty},
\end{align*}
where the optimal solution is denoted by $\bm \mu^*$.
\begin{align*}
    W_2^{\min} \ \coloneqq \ \min_{\bm C} \ \sum_{t=1}^T \|\bm C_t - \bm C \|_{1},
\end{align*}
where the optimal solution is denoted by $\bm C^*$.

\begin{claim}
The optimal solutions of the two optimization problems above must lie in the convex hull of $\{\bm \mu_t\}$ and $\{\bm C_t\}$, i.e.
$$ \bm \mu^* \in \mathrm{conv}(\{\bm \mu_1, \dots, \bm \mu_T\}), \quad \bm C^* \in \mathrm{conv}(\{\bm C_1, \dots, \bm C_T\}). $$
\end{claim}

\begin{claim}
Any parameter pair $(\tilde{\bm \mu}, \tilde{\bm C})$ that lies in the convex hulls of $\{\bm \mu_t\}$ and $\{\bm C_t\}$ must satisfy the dual price upper bound condition, i.e.
$$ \|\tilde{\bm q}^{*} \|_{\infty} \leq \max_{t} \|\bm q^*_t \|_{\infty} \leq \bar{q},\quad \forall \tilde{\bm \mu} \in \mathrm{conv}(\{\bm \mu_1, \dots, \bm \mu_T\}), \ \tilde{\bm C}\in \mathrm{conv}(\{\bm C_1, \dots, \bm C_T\})$$
where $\tilde{\bm q}^{*}$ is the dual optimal solution of the LP$(\tilde{\bm \mu}, \tilde{\bm C})$.
\end{claim}

\begin{proposition}
When the above two claims hold, the terms $W_1$ and $W_2$ in the regret bound of Theorem \ref{thm:upper} can be replaced by $W_{1}^{\min}$ and $W_{2}^{\min}$. 
\end{proposition}

The proof of the proposition is based on the arguments above, by replacing $W_1$ and $W_2$ in Proposition \ref{prop:LP} with $W_{1}^{\min}$ and $W_{2}^{\min}$.

\begin{figure}[ht!]
\centering
\includegraphics[scale=0.6]{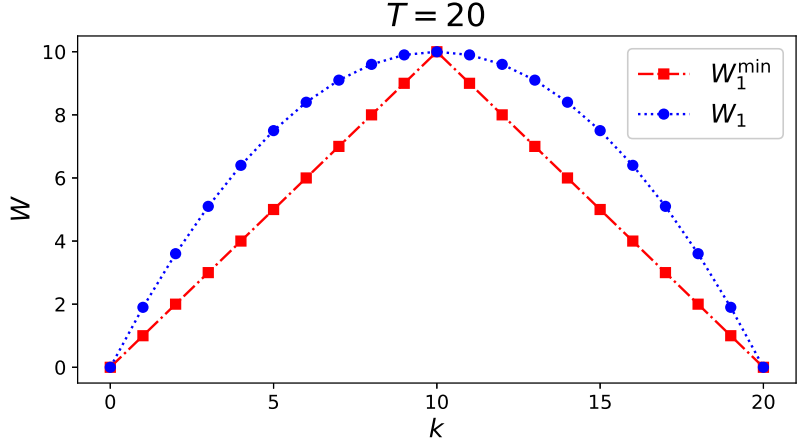}
\caption{Illustration of $W_1$ and $W_1^{\min}$.}
\label{fig:W-k}
\end{figure}

Figure \ref{fig:W-k} provides an illustration of the difference between $W_1$ (or $W_2$) and $W_1^{\min}$ (or $W_2^{\min}$) on a one-armed problem instance. Specifically, consider $\mu_1 = \cdots = \mu_k = 1$ and $\mu_{k+1} = \cdots = \mu_T = 0$ for some $k$. The optimal choice of $\mu^*$ is to set $\mu^* = 1$ for $k > \frac{T}{2}$ and $\mu^* = 0$ for $k < \frac{T}{2}$, resulting $W_1^{\min} = \min\{k, T - k\}$. Figure \ref{fig:W-k} plots the two non-stationarity measure against the change point $k.$ The result is not contradictory to the lower bound result in that when $k=\frac{T}{2}$, two definitions coincide with the same value.

\section{Discussions on the Non-bindingness of the Constraints}

\label{sec_binding}

As discussed in Section \ref{sec_sl}, the problem (and the regret bound) of non-stationary BwK can degenerate into the problem of (and the regret bound) non-stationary MAB when all the resource constraints are non-binding. In this section, we elaborate on the statement for two cases: (i) the benchmark LP and all the single-step LPs have all constraints non-binding; (ii) the benchmark LP have all constraints non-binding but some single-step LPs have some binding constraints. Throughout this section, we focus our discussion for the case when $V_1, V_2>0.$

\subsection{Non-binding for All LPs}

When the benchmark LP and all the single-step LPs have all constraints non-binding, our regret bound will reduce to the regret bound for non-stationary MAB \citep{besbes2014stochastic}. Specifically, when all the LPs have only non-binding constraints, we have
$$\bm q_t^* = \bm q^* = 0.$$
Hence $\bar{q} = 0$ and as a result, the regret bound in Theorem \ref{thm:upper} becomes
\begin{align*}
    \mathrm{Reg}(\mathrm{Alg}, T) & = O(\frac{1}{b}\cdot \sqrt{m T} \log (mdT^3) + m^{\frac{1}{3}} V_1^{\frac{1}{3}} T^{\frac{2}{3}}  \log^{\frac{1}{3}}(mT^3) + \frac{1}{b} \cdot m^{\frac{1}{3}} V_2^{\frac{1}{3}} T^{\frac{2}{3}}  \log^{\frac{1}{3}}(mdT^3) + W_1)\\
    & = \tilde{O}(m^{\frac{1}{3}} V_1^{\frac{1}{3}} T^{\frac{2}{3}}  + \frac{1}{b} \cdot m^{\frac{1}{3}} V_2^{\frac{1}{3}} T^{\frac{2}{3}}  + W_1), \quad \forall V_1, V_2 > 0.
\end{align*}
Compared to the regret bound for non-stationary MAB, there are still two additional terms (the terms of $V_2$ and $W_1$). In what follows, we discuss how to remove these two terms in the analysis. 

First, it is easy to get rid of the term related to $W_1$. In Proposition \ref{prop:LP}, we make use of $\mathrm{LP}(\bar{\bm \mu}, \bar{\bm C})$ as a bridge to relates $\sum_{t=1}^T \mathrm{LP}(\bm \mu_t, \bm C_t)$ and $\mathrm{LP}(\{\bm \mu_t\}, \{\bm C_t\}, T)$, and this causes the term of $W_1.$ When all the constraints are binding for all the LP's, then the following equality naturally holds
$$ \sum_{t=1}^T\mathrm{LP}(\bm \mu_t, \bm C_t) = \mathrm{LP}(\{\bm \mu_t\}, \{\bm C_t\}, T). $$
As a result, the analysis no longer needs $\mathrm{LP}(\bar{\bm \mu}, \bar{\bm C})$ as a bridge any more, which removes the $W_1$ term.

Second, for the term related to $V_2,$ we first point out that this term comes from the bound on the stopping time in Corollary \ref{coro:terminate}. When the single-step LPs have only non-binding constraints, and the non-bindingness remains stable with right-hand-side being $b'<b,$ we can remove this term $V_2.$ By remaining stable with $b'$, we mean the single-step LPs have only non-binding constraints if we replace the right-hand-side of the constraints $b$ with $b'.$ In this case, for a sufficiently large $T$ such that
$$(b-b^\prime)T \ge 4 \sqrt{T} \log(12mdT^3) + 9 m^{\frac{1}{3}} V_2^{\frac{1}{3}} T^{\frac{2}{3}} \log^{\frac{1}{3}} (12mdT^3),$$
we can apply the arguments in Lemma \ref{lemma:reward_consumption} to show that the stopping time $\tau\ge T$ with high probability. And thus we get rid of the $V_2$ term.

\subsection{Only Benchmark LP Non-binding}

When the benchmark LP is non-binding but the single-step LP has binding constraints, we show that the regret bound cannot be reduced to the case of non-stationary MAB. Specifically, consider a one-armed bandit problem instance with an even $T$. There are two types of resources, where each kind is of $\frac{2T}{3}$ budget. For the first half time periods, the arm has reward $1$ and consumes $1$ unit of resource $1$. For the second half, the arm still has reward $1$ but consumes $1$ unit of resource $2$ instead. The global LP is of course nonbinding with $\frac{T}{2} < \frac{2T}{3}$, while the one-step LPs are all binding with $\frac{2}{3} < 1$. The problem instance is in a similar spirit as the motivating example in Section \ref{motivate_eg}. In this case, $V_1=W_1=0$ but $\bar{q}>0$, the terms related to $V_2$ and $W_2$ cannot be removed. 

\section{Matching the Existing Bound of Stochastic (Stationary) BwK}

\label{sec_Alternative}

In this section, we will provide an alternative way to define the upper and lower confidence bounds using a slightly different concentration inequality. As a result, we derive an alternative regret upper bound for the non-stationary BwK problem, which matches the bound \citep{agrawal2014bandits} when the environment becomes stationary.

We first state the concentration inequality used in the previous works as a replacement of Lemma \ref{lemma:concentration}. 

\begin{lemma}\citep{kleinberg2008multi,babaioff2015dynamic,badanidiyuru2013bandits,agrawal2014bandits}.
\label{lemma:concentration_alter}
Consider some distribution with values in $[0,1]$. Denote its expectation by $z$. Let $\bar{Z}$ be the average of $N$ independent samples from this distribution. Then, $\forall \gamma > 0$, the following inequality holds with probability at least $1 - e^{\Omega (\gamma)}$,
$$ |\bar{Z} - z| \leq \mathrm{rad}(\bar{Z}, N) \leq 3 \mathrm{rad}(z, N), $$
where $\mathrm{rad}(a, b) \coloneqq \sqrt{\frac{\gamma a}{b}} + \frac{\gamma}{b}$. More generally, this result holds if $Z_1, \dots, Z_N \in [0,1]$ are random variables, $N \bar{Z} = \sum_{t=1}^N Z_t$, and $N z = \sum_{t=1}^N \mathbb{E}[Z_t|Z_1, \dots, Z_{t-1}]$.
\end{lemma}

Then we can revise confidence bounds in Algorithm \ref{alg:sl} as follows:
$$ \mathrm{UCB}_{t,i}(\bm \mu_t) \coloneqq \hat{\mu}_{t,i}^{(w_1)} + 2\cdot \mathrm{rad}(\hat{\mu}_{t,i}^{(w_1)}, n_{t,i}^{(w_1)} + 1), $$
$$ \mathrm{LCB}_{t,j,i}(\bm C_t) \coloneqq \hat{C}_{t,j,i}^{(w_{2})} - 2\cdot \mathrm{rad}(\hat{C}_{t,j,i}^{(w_{2})}, n_{t,i}^{(w_{2})} + 1), $$
where we will choose the window sizes $w_1,w_{2}$ according to $V_1, V_{2}$.

With the new concentration inequality and confidence bounds, Lemma \ref{lemma:reward_consumption} can be replaced by the following two lemmas. 

\begin{lemma}
With probability at least $1 - \frac{1}{T}$, we have
\begin{align*}
    \left|\sum_{t=1}^T (r_t - \mathrm{UCB}_t(\bm \mu_t)^\top \bm x_t)\right| &\leq O\left( \sqrt{\log(mT^2) \sum_{t=1}^T r_t} + \sqrt{\log(mT^2) m} \cdot \frac{T}{\sqrt{w_1}}+ w_1 V_1 + \log(mT^2)\right).
\end{align*}
If $V_1 > 0$ and we set $w_1 = \Theta(m ^{\frac{1}{3}}V_1^{-\frac{2}{3}} T^{\frac{2}{3}} \log^{\frac{1}{3}}(mT^2) )$, then
$$ \left|\sum_{t=1}^T (r_t - \mathrm{UCB}_t(\bm \mu_t)^\top \bm x_t )\right| = O\left(\sqrt{\log(mT^2) \sum_{t=1}^T r_t} + O(m ^{\frac{1}{3}} V_1^{\frac{1}{3}} T^{\frac{2}{3}} \log^{\frac{1}{3}}(mT^2)\right) \coloneqq \beta_1, $$
with probability at least $1 - \frac{1}{T}$.
\end{lemma}

\begin{lemma}
With probability at least $1 - \frac{1}{T}$, we have $\forall j$,
$$ \left|\sum_{t=1}^T (c_{t,j} - \mathrm{LCB}_t(\bm C_{t,j})^\top \bm x_t)\right| \leq O( \sqrt{\log(mdT^2) B_j} + \sqrt{\log(mdT^2) m} \cdot \frac{T}{\sqrt{w_{2}}} + w_{2} V_2 + \log(mdT^2). $$
If $V_{2} > 0$ and we set $w_{2} = \Theta(m ^{\frac{1}{3}} V_{2}^{-\frac{2}{3}} T^{\frac{2}{3}} \log^{\frac{1}{3}}(mdT^2) )$, then
$$ \left|\sum_{t=1}^T (c_{t,j} - \mathrm{LCB}_t(\bm C_{t,j})^\top \bm x_t )\right| = O(\sqrt{\log(mdT^2) B}) + O(m ^{\frac{1}{3}} V_{2}^{\frac{1}{3}} T^{\frac{2}{3}} \log^{\frac{1}{3}}(mdT^2) ) \coloneqq \beta_{2}, $$
with probability at least $1 - \frac{1}{T}$.
\end{lemma}

Following \citep{agrawal2014bandits}, one can fulfill the stopping time analysis by shrinking the resource budget in $\mathrm{LP}(\mathrm{UCB}_t(\bm \mu_t), \mathrm{LCB}_t(\bm C_t))$ by $\epsilon$. By choosing an appropriate $\epsilon \geq \frac{\beta_{2}}{B}$, we can show that the stopping criteria will not be met before $T$ with a high probability, since
\begin{align*}
\sum_{t=1}^T c_{t,j} & \leq \sum_{t=1}^T\mathrm{LCB}_t(\bm C_{t,j}^\top \bm x_t) + \beta_{2}\\
& \leq (1-\epsilon) B + \beta_{2}\\
& \leq B.
\end{align*}

In fact, if we take the stationary case into consideration as well, $\epsilon$ should be slightly larger. Here we use the concentration results in \citep{agrawal2014bandits} for the stationary cases directly:
\begin{lemma}[Lemma B.4 in \cite{agrawal2014bandits}]
If $V_1 = 0$ and we set $w_1 = T$, then
$$ \left|\sum_{t=1}^T (r_t - \mathrm{UCB}_t(\bm \mu_t)^\top \bm x_t )\right| = O(\sqrt{\log(mT^2) m \sum_{t=1}^T r_t}) + O(m \log(mT^2) ) \coloneqq \alpha_1, $$
with probability at least $1 - \frac{1}{T}$.
\end{lemma}

\begin{lemma}[Lemma B.5 in \cite{agrawal2014bandits}]
If $V_{2} = 0$ and we set $w_{2} = T$, then
$$ \left|\sum_{t=1}^T (c_{t,j} - \mathrm{LCB}_t(\bm C_{t,j})^\top \bm x_t )\right| = O(\sqrt{\log(mdT^2) m B}) + O(m\log(mdT^2)) \coloneqq \alpha_{2}, \quad \forall j$$
with probability at least $1 - \frac{1}{T}$.
\end{lemma}

One can choose
$$ \epsilon = \frac{\alpha_{2} + \beta_{2}}{B}, $$
so that the requirement is met. The shrunken LP will decrease the LP values up to $(1-\epsilon)$.

Then the final result goes as follows:
\begin{theorem}
\label{thm:refined_upper}
For the non-stationary bandits with knapsacks (NBwK) problem, if $B$ is not too small, i.e.
$$ m ^{\frac{1}{3}} V_2^{\frac{1}{3}} T^{\frac{2}{3}} \log^{\frac{1}{3}}(mdT^2)  = O(B), \quad \log(mdT^2) m = O(B), $$
and the dual prices are upper bounded by $\bar{q}$, then with probability at least $1-\frac{1}{T}$, the regret of the refined sliding-window confidence bound algorithm with $w_1$, $w_{2}$, and $\epsilon$ selected as suggested (denoted by $\pi_3$)
is upper bounded as
\begin{align*}
    \mathrm{Reg}(\pi_3, T) &= O(\ (\sqrt{\frac{m}{B}} \mathrm{OPT}(T) + \sqrt{m \mathrm{OPT}(T)} + m \sqrt{\log(mdT^2)} \ )\cdot \sqrt{\log (mdT^2)} \\
    &\phantom{O=} + m^{\frac{1}{3}} V_{1}^{\frac{1}{3}} T^{\frac{2}{3}} \sqrt[3]{\log (mT^2)} + \bar{q}  d m^{\frac{1}{3}} V_{2}^{\frac{1}{3}} T^{\frac{2}{3}}\sqrt[3]{\log (mdT^2)} \cdot   \\
    &\phantom{O=}+ W_1 + \bar{q} W_2).
\end{align*}
\end{theorem}

The result meets the upper bounds in \citep{agrawal2014bandits} and \citep{badanidiyuru2013bandits} up to logarithmic factors when the environment becomes stationary, i.e., $V_1=V_2=W_1=W_2=0$. In addition, the regret bound is expressed in terms of OPT$(T).$ Also note that this upper bound in Theorem \ref{thm:refined_upper} does not rely on the linear growth assumption (Assumption \ref{assumption:linear}), but it requires that $B = \Omega(V_2^{\frac{1}{3}}T^{\frac{2}{3}})$ at least for $V_2>0$.

\section{Proofs of Section \ref{sec_exten}}

\label{sec_proof_OCO}

Algorithm \ref{alg:virtualqueue} describes the Virtual Queue algorithm by \cite{neely2017online}. We first examine its performance with respect to deterministically adversarial constraints $g_{t,i}$'s. We emphasize that in this case, the nature is allowed to choose $g_{t,i}$'s after observing the player's decisions $\{x_1,\dots,x_{t-1}\}$ as long as the global non-stationarity budget is not violated.

\begin{algorithm}[ht!]
\caption{Virtual Queue Algorithm for OCOwC \cite{neely2017online}}
\label{alg:virtualqueue}
\begin{algorithmic}[1]
\Require Initial decision $\bm x_0$. Time horizon $T$. Parameters $\beta \gets 1/\sqrt{T}, \alpha \gets 1/T$.
\Ensure Decision sequence $\{\bm x_t\}$.
\State Initialize decision $\bm x_0 \gets \bm x_0$. Initialize virtual queue $Q_i(0) \gets 0, \ Q_i(1) \gets 0$.
\While{$1 \leq t \leq T$}
    \State Update virtual queue length if $t\geq 2$:        $$ Q_i(t) \gets \max\left\{0,Q_i(t-1) + g_{t-2,i}(\bm x_{t-2}) + \nabla g_{t-2,i}(\bm x_{t-2})^\top (\bm x_{t-1} - \bm x_{t-2})\right\}. $$
    \State Choose $\bm x_t$ as the solution of
        $$ \argmin_{\bm x\in \mathcal{X}} \left[\beta \nabla f_{t-1}(\bm x_{t-1})+ \sum_{i=1}^d Q_i(t)\nabla g_{t-1,j}(\bm x_{t-1})\right]^\top \bm x + \alpha \|\bm x - \bm x_{t-1}\|_2^2. $$
    \State Observe $\nabla f_t(\bm x_t), \nabla g_{t,i}(\bm x_t), i \in [d]$.
\EndWhile
\end{algorithmic}
\end{algorithm}

We first note that the benchmark taken into consideration in \cite{neely2017online} is $\mathrm{OPT}^\prime (T)$, which is based on a more restricted feasible solution set 
$$\mathcal{A}^\prime \coloneqq \{\bm x \in \mathcal{X}: \forall t \in [T], \ i \in [d],\  g_{t,i}(\bm x) \leq 0\}$$
compared to the feasible solution set considered in our paper:
$$\mathcal{A} \coloneqq \{\bm x \in \mathcal{X}: \forall i \in [d], \ \sum_{t=1}^T g_{t,i}(\bm x) \leq 0\}.$$
Thus our analysis complements the results therein with a more natural benchmark given the global nature of the constraints.

\subsection{Proof of Proposition \ref{prop:exten}}
\label{apd:prop_exten}

\begin{proof}
It follows directly from the fact that $\mathcal{A}^\prime \subset \mathcal{A}$ that 
$$\mathrm{OPT}^\prime (T) \geq \mathrm{OPT}(T).$$

We first prove the upper bound.
Recall that we denote the primal optimal solution of the standard optimization problem by $\bm x^*$ and that of the restricted optimization problem by $\bm x^{*'}$. Then $\bm x^*$ is a primal feasible solution to the perturbed restricted optimization problem
\begin{align*}
    \min_{\bm x\in \mathcal{X}} \ & \sum_{t=1}^T f_t(\bm x)\\
    \text{s.t. }& g_{t,i}(\bm x) \leq (g_{t,i}(\bm x^*))^+, \quad \forall t \in [T],\  i \in [d].
\end{align*}
Denote the optimal value of the perturbed problem by $\mathrm{OPT}^{\prime \prime}(T)$.

Since we have assumed the Slater's condition in Assumption \ref{assumption:slater}, the strong duality holds, and we have
\begin{align*}
    \mathrm{OPT}^\prime(T) - \mathrm{OPT}(T) & \leq \mathrm{OPT}^\prime(T) - \mathrm{OPT}^{\prime \prime}(T)\\
    &\leq \bar{q} \sum_{t=1}^T \sum_{i=1}^d (g_{t,i}(\bm x^*))^+\\
    &= \bar{q} \sum_{i=1}^d \sum_{t=1}^T (g_{t,i}(\bm x^*))^+\\
    &\leq \bar{q} \sum_{i=1}^d (T (\bar{g}_i(\bm x^*))^+ + \sum_{t=1}^T \|g_{t,i} - \bar{g}_i\|_{\infty})\\
    &= \bar{q} \sum_{t=1}^T\sum_{i=1}^d \|g_{t,i}- \bar{g}_i\|_{\infty} \\
    &= \bar{q} W.
\end{align*}

As for the lower bound, we construct an example based on the idea of the third example given in Theorem \ref{thm:lower}. The decision set $\mathcal{X}$ is $[0,1]$. We set the target functions and constraint functions as
$$ f_t(x) = -r x, $$
$$ g_t(x) = (b + \Delta \mathbbm{1}\{t \leq \frac{T}{2}\}) x - \frac{1}{2}b. $$
Then the example is now an OCOwC instance with global non-stationarity budgets $W = \Theta(\Delta T)$ and $\bar{q} = \frac{r}{b}$. The restricted optimal value $\mathrm{OPT}^\prime(T)$ is now
$$ \mathrm{OPT}^\prime(T) = -(\frac{T}{4} \cdot \frac{r b}{b + \Delta} + \frac{T}{4} \cdot r),$$
while the standard optimal value $\mathrm{OPT}(T)$ is now
$$ \mathrm{OPT}(T) = -\frac{T}{2} \cdot r. $$
If we assume that $\Delta = o(b)$ (which can always be satisfied if we let $b$ to be sufficiently small), then
$$ \mathrm{OPT}^\prime (T) - \mathrm{OPT}(T) = \frac{T}{4} \cdot r \cdot \frac{\Delta}{b + \Delta} = \Omega(\bar{q} W). $$
\end{proof}

\subsection{Proof of Theorem \ref{thm:exten}}

The Virtual Queue algorithm (Algorithm \ref{alg:virtualqueue}) proposed by \cite{neely2017online} incurs at most $O(\sqrt{T})$ expected regret against the restricted static benchmark $\mathrm{OPT}^\prime(T)$ under Assumption \ref{assumption:slater}. Furthermore, the expected overall constraints violation is bounded by $O(\sqrt{T})$ for each $i \in [d]$ (see detailed analysis in Theorem 1, 3, 4 in \cite{neely2017online}). While their analysis is against the benchmark $\mathrm{OPT}^\prime(T)$, we here present a result against the stronger benchmark $\mathrm{OPT}(T)$ (Theorem \ref{thm:exten}).

\begin{proof}
As is shown in Theorem 1 in \cite{neely2017online}, the Algorithm \ref{alg:virtualqueue} achieves
$$ \sum_{t=1}^T f_t(\bm X_t) \leq \mathrm{OPT}^\prime (T) + O(\sqrt{T}). $$
By Proposition \ref{prop:exten},
$$ \mathrm{OPT}^\prime - \mathrm{OPT}(T) \leq \bar{q} W. $$
Combining above two inequalities together, we have
$$ \mathrm{Reg}_1(\pi_3, T) \leq O(\sqrt{T}) + \bar{q}W. $$
As for constraint violation, we shall directly apply Theorem 3 in \cite{neely2017online} such that
$$ \mathrm{Reg}_2(\pi_3, T) = \sum_{i=1}^d (\sum_{t=1}^T g_{t,i}(\bm X_t))^+ \leq d O(\sqrt{T}). $$
\end{proof}

\subsection{Extension to the Stochastic Setting}

The results in Theorem \ref{thm:exten} can be further extended to a stochastic setting where the adversary is oblivious of our decisions. That is, the distributions that govern the random functions $g_{t,i}$ can be chosen adversarially in advance but cannot be adaptively changed according to the decisions $\bm{x}_t$'s/ 

We modify the performance measures accordingly as follows
$$ \mathrm{Reg}_1(\pi, T) \coloneqq \mathbb{E}[\sum_{t=1}^T f_t(\bm X_t) - \sum_{t=1}^T f_t(\bm x^{*'})], $$
$$ \mathrm{Reg}_2(\pi, T) \coloneqq \sum_{i=1}^d \mathbb{E}[(\sum_{t=1}^T g_{t,i}(\bm X_t))^+], $$
where $\bm x^{*'}$ is the minimizer of $\sum_{t=1}^T f_t(\bm x)$ on stochastic feasible set
$$ \left\{\bm x \in \mathcal{X}: \E[\sum_{t=1}^T g_{t,i}(\bm x)] \leq 0\right\}. $$
We define the certainty equivalent convex programs by
\begin{align*}
    \mathrm{OPT}(T) \coloneqq \ \min_{\bm{x}\in\mathcal{X}} \ & \sum_{t=1}^T f_t(\bm{x}) \\
    \text{s.t. }&\sum_{t=1}^T  \E[g_{t,i}(\bm{x})] \le 0, \text{ for } i \in[d],
\end{align*}
and
\begin{align*}
    \mathrm{OPT}^\prime(T) \coloneqq \ \min_{\bm{x}\in\mathcal{X}} \ & \sum_{t=1}^T f_t(\bm{x}) \\
    \text{s.t. }& \E[g_{t,i}(\bm{x})] \le 0, \text{ for } t \in [T], \ i \in[d].
\end{align*}
Note that Slater's condition in deterministic case can be relaxed to stochastic Slater's condition according to \cite{yu2017online}, i.e.
$$ \exists \bm x, \text{s.t. } \E[g_{t,i}(\bm x)] < 0, \quad \forall t, i. $$

There are some algorithms that are of the same type as the Virtual Queue Algorithm  in \cite{neely2017online} specifically designed for the case with i.i.d. $\bm g_t$'s (see \cite{yu2017online} and \cite{wei2020online}). To obtain similar $O(\sqrt{T})$ regret bound for the stochastic setting, aforementioned papers utilized some similar lemmas (Lemma 6 in \cite{neely2017online}, Lemma 6 in \cite{yu2017online}, and Lemma 8 in \cite{wei2020online}) which guarantee that
$$ \mathbb{E}\left[\sum_{i=1}^d Q_{i}(t) g_{t-1,i}(\bm x^{*'})\right] \leq 0. $$
For deterministic cases, the above conclusion is reduced to
$$ \sum_{i=1}^d Q_{i}(t) g_{t-1,i}(\bm x^{*'}) \leq 0, $$
which automatically holds by the fact that $Q_i(t) \geq 0$ and $g_{t,i}(\bm x^{*'}) \leq 0$. But for stochastic cases, the proof becomes trickier, since we relax the condition $g_{t,i}(\bm x^{*'}) \leq 0$ to $\E[g_{t,i}(\bm x^{*'})] \leq 0$.

In the aforementioned papers, the lemma is proved via factorization of the expectations. Since $Q_{i}(t)$ is determined by the previous $1\leq s \leq t-2$ steps' $f_s$'s and $g_{s,i}$'s (note that $\bm X_{t-1}$ is determined by previous $t-2$ steps' functions) while $\bm g_{t-1}$ is independent of previous $t-2$ steps, one can factorize the expectations so that the lemma holds. Specifically, we define two random processes $\{\xi^{t}\}_{t=1}^\infty$ and $\{\gamma^{t}\}_{t=1}^\infty$ such that $f_t(\bm x) = f(\bm x, \xi^t)$ and $\bm g_{t}(\bm x) = \bm g(\bm x, \gamma^t)$. We define a filtration $\{\mathcal{F}_t: t\geq 0\}$ with $\mathcal{F}_t \coloneqq \{\xi^\tau, \gamma^\tau\}_{\tau=1}^{t-1}$. Then taking conditional expectations of $\sum_{i=1}^d Q_{i}(t) g_{t-1,i}(\bm x^{*'})$ yields that
\begin{align*}
    \mathbb{E}[\sum_{i=1}^d Q_{i}(t) g_{t-1,i}(\bm x^{*'})|\mathcal{F}_{t-1}] & = \sum_{i=1}^d Q_{i}(t)\mathbb{E}[g_{t-1,i}(\bm x^{*'})|\mathcal{F}_{t-1}]\\
    & = \sum_{i=1}^d Q_{i}(t) \mathbb{E}[g_{t-1,i}(\bm x^{*'})] \leq 0
\end{align*}
where the last equality holds since $\bm g_t$'s are assumed to be i.i.d. in aforementioned papers.

From the above discussions, we can see that the result still holds without the i.i.d. assumption as long as 
$$ \E[\bm g_t | \mathcal{F}_{t}] = \E[\bm g_t]. $$
Such a requirement is automatically fulfilled when all $\bm g_t$'s are assumed to be distributed independent of previous $\{f_\tau, \bm g_\tau\}_{\tau=1}^{t-1}$'s. Then one can easily derive that the Virtual Queue Algorithm \ref{alg:virtualqueue} (denoted by $\pi_2$) induces a result of
$$ \mathrm{Reg}_1(\pi_2, T) \leq O(\sqrt{T}) + \bar{q}W, $$
$$ \mathrm{Reg}_2(\pi_2, T) \leq O(d\sqrt{T}), $$
where
$$ W \coloneqq \sum_{t=1}^T \sum_{i=1}^d \|\E[g_{t,i}] - \E[\bar{g}_j]\|_{\infty}, $$
and $\bar{q}$ is the upper bound for the optimal dual solutions of the certainty equivalent convex programs.

\end{document}